\documentclass[twoside,11pt]{article}

\usepackage{graphicx,subfig}
\usepackage{caption}
\usepackage[top=1in,bottom=1in,left=1in,right=1in]{geometry}
\usepackage[sort&compress, numbers]{natbib} 
\usepackage{amsfonts, amsmath, amssymb, amsthm}
\usepackage{algorithm, algorithmic}

\usepackage{color}
\definecolor{darkred}{RGB}{100,0,0}
\definecolor{darkgreen}{RGB}{0,100,0}
\definecolor{darkblue}{RGB}{0,0,150}

\usepackage{hyperref}
\hypersetup{colorlinks=true, linkcolor=darkred, citecolor=darkgreen, urlcolor=darkblue}

\def\one{1}
\def\M{\cM}

\newtheorem{thm}{Theorem}
\newtheorem{prp}{Proposition}
\newtheorem{lem}{Lemma}
\newtheorem{cor}{Corollary}
\theoremstyle{remark}
\newtheorem{Def}{Definition}
\newtheorem{rem}{Remark}
\newtheorem{example}{Example}

\def\beq{\begin{equation}} 
\def\eeq{\end{equation}}
\def\beqn{\beq\notag}
\def\Bitem{\begin{itemize}\setlength{\itemsep}{.2in}}
\def\bitem{\begin{itemize}\setlength{\itemsep}{.05in}}
\def\eitem{\end{itemize}}
\def\Benum{\begin{enumerate}\setlength{\itemsep}{.2in}}
\def\benum{\begin{enumerate}\setlength{\itemsep}{.05in}}
\def\eenum{\end{enumerate}}
\def\bmult{\begin{multline*}}
\def\emult{\end{multline*}}
\def\bcenter{\begin{center}}
\def\ecenter{\end{center}}
\def\bframe{\begin{frame}}
\def\eframe{\end{frame}}

\newcommand{\thmref}[1]{Theorem~\ref{thm:#1}}
\newcommand{\prpref}[1]{Proposition~\ref{prp:#1}}
\newcommand{\corref}[1]{Corollary~\ref{cor:#1}}
\newcommand{\lemref}[1]{Lemma~\ref{lem:#1}}
\newcommand{\secref}[1]{Section~\ref{sec:#1}}
\newcommand{\figref}[1]{Figure~\ref{fig:#1}}

\newcommand{\algref}[1]{Algorithm~\ref{alg:#1}}
\newcommand{\remref}[1]{Remark~\ref{rem:#1}}

\DeclareMathOperator{\diag}{diag}

\def\cA{\mathcal{A}}

\def\cD{\mathcal{D}}

\def\cL{\mathcal{L}}
\def\cM{\mathcal{M}}

\def\cO{\mathcal{O}}

\def\cY{\mathcal{Y}}
\def\cZ{\mathcal{Z}}

\def\bbI{\mathbb{I}}

\def\bbR{\mathbb{R}}

\def\avgc{c_{{\rm avg}}}

\renewcommand{\P}{\operatorname{\mathbb{P}}}

\def\eps{\varepsilon}

\def\1{\mathbbm{1}}

\def\tLambda{\tilde{\Lambda}}
\def\reals{\mathbb{R}}

\def\tdelta{\tilde{\delta}}
\def\tDelta{\tilde{\Delta}}
\def\ty{\tilde{y}}
\def\tz{\tilde{z}}
\def\tY{\tilde{Y}}
\def\tZ{\tilde{Z}}
\def\tlambda{\tilde{\lambda}}
\def\normal{{\sf N}}

\newcommand{\IND}[1]{\bbI\{ #1 \}}

\definecolor{purple}{rgb}{0.4,.1,.9}

\newcommand{\new}[1]{\textcolor{black}{#1}}
\newcommand{\ajr}[1]{{{\color{black}{#1}}}}
\newcommand{\ajR}[1]{{{\color{black}{#1}}}}

\begin{document}

\title{Perturbation Bounds for Procrustes, Classical Scaling, and Trilateration, with Applications to Manifold Learning}
\author{
Ery Arias-Castro\footnote{Department of Mathematics, University of California, San Diego}%
\and
Adel Javanmard\footnote{Marshall School of Business, University of Southern California}
\and
Bruno Pelletier\footnote{D\'epartement de Math\'ematiques, IRMAR - UMR CNRS 6625, Universit\'e Rennes II}
}
\date{\today}
\maketitle

\begin{abstract}\noindent
One of the common tasks in unsupervised learning is dimensionality reduction, where the goal is to find meaningful low-dimensional structures hidden in high-dimensional data.  Sometimes referred to as manifold learning, this problem is closely related to the problem of localization, which aims at embedding a weighted graph into a low-dimensional Euclidean space.  
Several methods have been proposed for localization, and also manifold learning. 
Nonetheless, the robustness property of most of them is little understood. In this paper, we obtain perturbation bounds for classical scaling and trilateration, which are then applied to derive performance bounds for Isomap, Landmark Isomap, and Maximum Variance Unfolding.  A new perturbation bound for procrustes analysis plays a key role.
\end{abstract}

\section{Introduction} \label{sec:intro}
Multidimensional scaling (MDS) can be defined as the task of embedding an itemset as points in a (typically) Euclidean space based on some dissimilarity information between the items in the set.  Since its inception, dating back to the early 1950's if not earlier \cite{young2013multidimensional}, MDS has been  one of the main tasks in the general area of multivariate analysis, a.k.a., unsupervised learning.  

One of the main methods for MDS is called classical scaling, which consists in first double-centering the dissimilarity matrix and then performing  an eigen-decomposition of the obtained matrix.  This is arguably still the most popular variant, even today, decades after its introduction at the dawn of this literature.
(For this reason, this method is often referred to as MDS, and we will do the same on occasion.)
Despite its wide use, its perturbative properties remain little understood.  The major contribution on this question dates back to the late 1970's with the work of \citet{sibson1979studies}, who performs a sensitivity analysis that resulted in a Taylor development for the classical scaling to the first nontrivial order.  
Going beyond \citet{sibson1979studies}'s work, our first contribution is to derive a bonafide perturbation bound for classical scaling (\thmref{classical scaling}).

Classical scaling amounts to performing an eigen-decomposition of the dissimilarity matrix after double-centering.  Only the top $d$ eigenvectors are needed if an embedding in dimension $d$ is desired.  Using iterative methods such as the Lanczos algorithm, classical scaling can be implemented with a complexity of $O(d n^2)$, where $n$ is the number of items (and therefore also the dimension of the dissimilarity matrix).
In applications, particularly if the intent is visualization, the embedding dimension $d$ tends to be small.  Even then, the resulting complexity is quadratic in the number of items $n$ to be embedded.
There has been some effort in bringing this down to a complexity that is linear in the number of items.  The main proposals \cite{faloutsos1995fastmap, wang1999evaluating, de2004sparse} are discussed by \citet{platt2005fastmap}, who explains that all these methods use a Nystr\"om approximation.  The procedure proposed by  \citet{de2004sparse}, which they called landmark MDS (LMDS) and which according to \citet{platt2005fastmap} is the best performing methods among these three, works by selecting a small number of items, perhaps uniformly at random from the itemset, and embedding them via classical scaling. These items are used as landmark points to enable the embedding of the remaining items.  The second phase consists in performing trilateration, which aims at computing the location of a point based on its distances to known (landmark) points. Note that this task is closely related to, but distinct, from triangulation, which is based on angles instead.
If $\ell$ items are chosen as landmarks in the first step (out of $n$ items in total), then the procedure has complexity $O(d \ell^2 + d \ell n)$.  Since $\ell$ can in principle be chosen on the order of $d$, and $d \le n$ always, the complexity is effectively $O(d^2 n)$, which is linear in the number of items.
A good understanding of the robustness properties of LMDS necessitates a good understanding of the robustness properties of not only classical scaling (used to embed the landmark items), but also of trilateration (used to embed the remaining items).
Our second contribution is a perturbation bound for trilateration (\thmref{trilateration}).  There are several closely related method for trilateration, and we study on the method proposed by \citet{de2004sparse}, which is rather natural.  We refer to this method simply as trilateration in the remaining of the paper.
 
\citet{de2004sparse} build on the pioneering work of \citet{sibson1979studies} to derive a sensitivity analysis of classical scaling.  They also derive a sensitivity analysis for their trilateration method following similar lines.  
In the present work, we instead obtain bonafide perturbation bounds, for procrustes analysis (\secref{procrustes}), for classical scaling (\secref{classical scaling}), and for the same trilateration method (\secref{trilateration}).  
In particular, our perturbation bounds for procrustes analysis and classical scaling appear to be new, which may be surprising as these methods have been in wide use for decades.  (The main reason for deriving a perturbation bound for procrustes analysis is its use in deriving a perturbation bound for classical scaling, which was our main interest.)
These results are applied in \secref{applications} to Isomap, Landmark Isomap, and also Maximum Variance Unfolding (MVU).  These may be the first performance bounds of any algorithm for manifold learning in its `isometric embedding' variant, even as various consistency results have been established for Isomap \cite{zha2007continuum}, MVU \cite{arias2013convergence}, and a number of other methods \cite{donoho2003hessian, ye2015discrete, gine2006empirical, smith2008convergence, belkin2008towards, vonLuxburg08, singer2006graph, hein2005graphs, coifman2006diffusion}.  (As discussed in \cite{goldberg2008manifold}, Local Linear Embedding, Laplacian Eigenmaps, Hessian Eigenmaps, and Local Tangent Space Alignment, all require some form of normalization which make them inconsistent for the problem of isometric embedding.)
In \secref{discussion} we discuss the question of optimality in manifold learning and also the choice of landmarks.
The main proofs are gathered in \secref{proofs}.

\section{A perturbation bound for procrustes}
\label{sec:procrustes}
The orthogonal procrustes problem is that of aligning two point sets (of same cardinality) using an orthogonal transformation.  In formula, given two point sets, $x_1, \dots, x_m$ and $y_1, \dots, y_m$ in $\bbR^d$, the task consists in solving
\beq
\min_{Q \in \cO(d)} \sum_{i=1}^m \|y_i - Q x_i\|^2,
\eeq
where $\cO(d)$ denotes the orthogonal group of $\bbR^d$.  (Here and elsewhere, when applied to a vector, $\|\cdot\|$ will denote the Euclidean norm.)

In matrix form, the problem can be posed as follows.  Given matrices $X$ and $Y$ in $\bbR^{m \times d}$, solve
\beq
\min_{Q \in \cO(d)} \|Y - X Q\|_2,
\eeq
where $\|\cdot\|_2$ denotes the Frobenius norm (in the appropriate space of matrices). 
As stated, the problem is solved by choosing $Q = U V^\top$, where $U$ and $V$ are $d$-by-$d$ orthogonal matrices obtained by a singular value decomposition of $X^\top Y = U D V^\top$, where $D$ is the diagonal matrix with the singular values on its diagonal \cite[Sec 5.6]{seber2009multivariate}. 
\algref{procrustes} describes the procedure. 

\begin{algorithm}[!ht]
\caption{Procrustes (Frobenius norm)}
\label{alg:procrustes}
\begin{algorithmic}
\STATE {\bf Input:} point sets $x_1, \dots, x_m$ and $y_1, \dots, y_m$ in $\bbR^d$
\STATE {\bf Output:} an orthogonal transformation $Q$ of $\bbR^d$
\medskip
\STATE {\bf 1:} store the point sets in $X = [x_1^\top \cdots x_m^\top]$ and $Y = [y_1^\top \cdots y_m^\top]$
\STATE {\bf 2:} compute $X^\top Y$ and its singular value decomposition $U D V^\top$
\STATE {\bf Return:} the matrix $Q = U V^\top$
\end{algorithmic}
\end{algorithm}

In matrix form, the problem can be easily stated using any other matrix norm in place of the Frobenius norm.  There is no closed-form solution in general, even for the operator norm (as far as we know), although some computational strategies have been proposed for solving the problem numerically
\cite{watson1994solution}.  In what follows, we consider an arbitrary Schatten norm.
For a matrix $A\in \bbR^{m\times n}$, let $\|A\|_p$ denote the Schatten $p$-norm, where $p \in [1, \infty]$ is assumed fixed:
\begin{align}
\|A\|_p \equiv \Big(\sum_{i\ge 1} \nu_i^p(A) \Big)^{1/p}\,,
\end{align}
with $\nu_1(A) \ge \nu_2(A) \ge \dotsc \ge 0$ the singular values of $A$.
Note that $\|\cdot\|_2$ coincides with the Frobenius norm. We also define $\|A\|_\infty$ to be the usual operator norm, i.e., the maximum singular value of $A$.  Henceforth, we will also denote the operator norm by $\|\cdot\|$, on occasion.
We denote by $A^\ddag$ the pseudo-inverse of $A$ (see \secref{preliminaries}). \ajR{Henceforth, we also use the notation $a\wedge b = \min(a,b)$ for two numbers $a, b$.}

\ajr{Our first theorem is a perturbation bound for procrustes, where the distance between two configurations of points $X$ and $Y$ is bounded in terms of the distance between their Gram matrices $XX^\top$ and $YY^\top$.}  

\ajR{
\begin{thm}
\label{thm:procrustes}
Consider two tall matrices $X$ and $Y$ of same size, with $X$ having full rank, and set $\eps^2 = \|Y Y^\top - X X^\top\|_p$. 
Then, we have
\begin{align}\label{mysharper}
\min_{Q \in \cO} \|Y - X Q\|_p \le \begin{cases}
\|X^\ddag\| \eps^2 + \left((1 - \|X^\ddag\|^2 \eps^2)^{-1/2} \|X^\ddag\| \eps^2\right) \wedge (d^{1/2p} \eps) \,, \quad &\text{ if } \|X^\ddag\|  \eps< 1\,,\\
\|X^\ddag\| \eps^2 + d^{1/2p} \eps &\text{ otherwise.}
\end{cases}
\end{align}
Consequently, if $\|X^\ddag\| \eps \le \frac1{\sqrt{2}}$, then 
\beq\label{procrustes}
\min_{Q \in \cO} \|Y - X Q\|_p \le (1+ \sqrt{2}) \|X^\ddag\| \eps^2.
\eeq
\end{thm}
}

The proof is in \secref{proof_procrustes}.
Interestingly, to establish the upper bound we use an orthogonal matrix constructed from the singular value decomposition of $X^\ddag Y$.  This is true regardless of $p$, which may be surprising since a solution for the \ajr{Frobenius} norm (corresponding to the case where $p = 2$) is based on a singular value decomposition of $X^\top Y$ instead.

\ajR{Also, let us stress that $\eps$ in the theorem statement, by definition, depends on the choice of $p$-norm.}

\begin{example}\label{example:orthogonal} \ajR{(\bf Orthonormal matrices)}
The case where $X$ and $Y$ are \ajR{orthonormal and of the same size} is particularly simple, at least when $p = 2$ or $p = \infty$, based on what is already known in the literature.  Indeed, from \cite[Sec II.4]{MR1061154} we find that, in that case, 
\beq
\min_{Q \in \cO} \|Y - X Q\|_p = \|2 \sin(\tfrac12 \theta(X, Y))\|_p,
\eeq
where $\theta(X,Y)$ is the diagonal matrix made of the principal angles between the subspaces defined by $X$ and $Y$, \ajr{and for a matrix $A$, $\sin(A)$ is understood entrywise. In addition,}
\beq
\eps^2 = \|YY^\top - XX^\top\|_p = \|\sin \theta(X, Y)\|_p.
\eeq
Using the elementary inequality $\sqrt{2} \sin(\alpha/2)\le \sin (\alpha) \le 2\sin(\alpha/2)$, valid for $\alpha\in [0,\pi/2]$, we get
\beq
\eps^2 \le \min_{Q \in \cO} \|Y - X Q\|_p \le \sqrt{2} \eps^2.
\eeq
Note that, in this case, $\|X\| = \|X^\ddag\| = 1$, and \ajR{our bound \eqref{procrustes} gives the upper bound $(1+\sqrt{2})\eps^2$, which is tight up to a factor of $1+\tfrac{1}{\sqrt{2}}$.}
\end{example}

\begin{example}\label{example:diagonal}
The derived perturbation bound~\eqref{procrustes} includes the pseudo-inverse of the configuration, $\|X^\ddagger\|$. Nonetheless, \ajR{the example of orthogonal matrices}
does not capture this factor because $\|X^\ddagger\| = 1$ in that case. To build further insight on our result in \thmref{procrustes}, we consider another example where $X$ and $Y$ share the same singular vectors. Namely $X = U\Lambda V^\top$ and $Y = U\Theta V^\top$, with $U\in \bbR^{m\times d}$, $V\in \bbR^{d\times d}$ orthonormal matrices, and $\Lambda = \diag(\{\lambda_i\})_{i=1}^d$ and $\Theta = \diag(\{\theta_i\})_{i=1}^d$. 
Consider the case of $p=2$, and let $X^\top Y = V (\Lambda \Theta)V^\top$ be a singular value decomposition. Then by \algref{procrustes}, the optimal
rotation is given by $Q = I$. We therefore have
\begin{align}
 \min_{Q \in \cO} \|Y - X Q\|_2  &= \Big[\sum_{i\in[n]} (\theta_i - \lambda_i)^2\Big]^{1/2} =\Big[\sum_{i\in[n]} \Big(\frac{\theta_i^2 - \lambda_i^2}{\theta_i+\lambda_i}\Big)^2\Big]^{1/2}\label{eq:l1}\\
 &\le \frac{1}{(\underset{i\in [n]}{\min} \,|\lambda_i|)} \Big[\sum_{i\in[n]} \Big({\theta_i^2 - \lambda_i^2}\Big)^2\Big]^{1/2} =  \frac{1}{(\underset{i\in [n]}{\min} \,|\lambda_i|)} \|YY^\top - XX^\top\|_2\nonumber\\
 & = \|X^\ddagger\| \eps^2\,.\label{eq:U-diagonal}
\end{align}
 Let us stress that the above derivation applies only to this example, but it showcases the relevance of $\|X^\ddagger\|$ in the bound. 
 
 We next develop a lower bound for the following specific case. \ajR{Let $D = \diag(1, 1,\dotsc, \delta)$ for arbitrary but fixed $\delta\in [0,1]$ and let $\Theta = \diag(1, 1, \dotsc, \sqrt{\delta^2+\eps^2})$. Then, $\|X^\ddagger\| = 1/\delta$ and  $\|XX^\top- YY^\top\|_2 = \eps^2$. By~\eqref{eq:l1} we have
 \begin{align}\label{eq:L-diagonal0}
 \min_{Q \in \cO} \|Y - X Q\|_2 =  \Big[\sum_{i\in[n]} (\theta_i - \lambda_i)^2\Big]^{1/2} = \sqrt{\delta^2+\eps^2}-\delta
  = \delta\left(\sqrt{1+\tfrac{\eps^2}{\delta^2}} - 1\right)\,.
  \end{align}
  Also, from the condition $\|X^\ddag\| \eps \le \frac1{\sqrt{2}} $ we have $\tfrac{\eps}{\delta}<\tfrac{1}{\sqrt{2}}$. Using $\sqrt{1+x^2} -1\ge (\sqrt{6}-2)x^2$, which holds for $x<\tfrac{1}{\sqrt{2}}$ and substituting for $\delta  = 1/\|X^\ddagger\|$, we obtain
  \begin{align}\label{eq:L-diagonal}
   \min_{Q \in \cO} \|Y - X Q\|_2 \ge (\sqrt{6} -2)\|X^\ddagger\| {\eps^2}  
  \end{align}
 From \eqref{eq:U-diagonal} and \eqref{eq:L-diagonal}, we observe that the $\|X^\ddagger\|$ term appears both in the upper and the lower bounds of the procrustes error, which confirms its relevance.
 }
\end{example}

\ajr{
\begin{rem}
We emphasize that the general bound in\eqref{mysharper} does not require any restriction on $\eps$.  However, as it turns out, the result in \eqref{procrustes} would be already enough for our purposes in the next sections and deriving our results in the context of manifold learning. Regarding the procrustes error bound in \thmref{procrustes}, we do conjecture that there is a smooth transition between a bound in $\eps^2$ and a bound in $\eps$ as $\|X^\ddag\|$ increases to infinity (and therefore $X$ degenerates to a singular matrix).  For instance, in Example~\ref{example:diagonal}, when $\delta  = \|X^\ddagger\|^{-1} \to 0$ faster than $\eps$, the lower bound~\eqref{eq:L-diagonal0} scales linearly in $\eps$.
\end{rem}
}

\ajR{It is worth noting that other types of perturbation analysis have been carried out for the procrustes problem. For example~\cite{soderkvist1993perturbation} considers the procrustes problem over the class of rotation matrices, a subset of orthogonal matrices, and study how its solution (optimal rotation) would be perturbed if both configurations were perturbed. In \cite{zha2009spectral}, the authors study the perturbation of the null space of a similarity matrix from manifold learning, using the standard perturbation theory for invariant subspaces~\cite{MR1061154}.}

\section{A perturbation bound for classical scaling}
\label{sec:classical scaling}

In multidimensional scaling, we are given a matrix, $\Delta = (\Delta_{ij}) \in \bbR^{m \times m}$, storing the dissimilarities between a set of $m$ items (which will remain abstract in this paper).  \ajr{A square matrix $\Delta$ is called dissimilarity matrix if it is symmetric, $\Delta_{ii} = 0$, and $\Delta_{ij} > 0$, for $i\neq j$. ($\Delta_{ij}$ gives the level of dissimilarity between items $i, j \in [m]$.) }
Given a positive integer $d$, we seek a configuration, meaning a set of points, $y_1, \cdots, y_m \in \bbR^d$, such that $\|y_i - y_j\|^2$ is close to $\Delta_{ij}$ over all $i,j \in [m]$.  
The itemset $[m]$ is thus embedded as $y_1, \dots, y_m$ in the $d$-dimensional Euclidean space $\bbR^d$.

\algref{classical scaling} describes classical scaling, the first practical and the most prominent method for solving this problem. The method is widely attributed to \citet{torgerson1958theory} \ajr{and Gower \cite{gower1966some}} and it is also known under the names Torgerson scaling and Torgerson-Gower scaling. 

\begin{algorithm}[!ht]
\caption{Classical Scaling}
\label{alg:classical scaling}
\begin{algorithmic}
\ajr{
\STATE {\bf Input:} dissimilarity matrix $\Delta \in \bbR^{m \times m}$, embedding dimension $d$
\STATE {\bf Output:} set of points $y_1, \dots, y_m \in \bbR^d$
\medskip
\STATE {\bf 1:} compute the matrix $\Delta^c = -\tfrac12 H \Delta H$
\STATE{\bf 2:} let $\lambda_1\ge \lambda_2\ge\dotsc \ge \lambda_m$ be the eigenvalues of $\Delta^c$, with corresponding eigenvectors $u_1,\dotsc, u_m$
\STATE {\bf 3:} compute $Y \in \bbR^{m \times d}$ as $Y = [\sqrt{\lambda_{1,+}}\;u_1, \dotsc, \sqrt{\lambda_{d,+}}\;u_d]$
\STATE {\bf Return:} the row vectors $y_1, \dots, y_m$ of $Y$
}
\end{algorithmic}
\end{algorithm}

\ajr{In the description, $H = I - J/m$ is the centering matrix in dimension $m$, where $I$ is the identity matrix and $J$ is the matrix of ones. Further, we use the notation $a_{+} = \max(a,0)$ for a scalar $a$.
The basic idea of classical scaling is to assume that the dissimilarities are Euclidean distances and then find coordinates that explain them.}

\ajr{For a general dissimilarity matrix $\Delta$, the doubly centered matrix $\Delta^c$ may have negative eigenvalues and that is why in the construction of $Y$, we use the positive part of the eigenvalues. However, if $\Delta$ is an Euclidean dissimilarity matrix, namely $\Delta_{ij} = \|x_i - x_j\|^2$ for a set points $\{x_1, \dotsc, x_m\}$ in some ambient Euclidean space, then $\Delta^c$ is a positive semi-definite matrix. This follows from the following identity relating a configuration $X$ with the corresponding squared distance matrix $\Delta$:}
\beq\label{eq:MDSidentity}
-\frac{1}{2} H \Delta H = H XX^T H\,.
\eeq

Consider the situation where the dissimilarity matrix $\Delta$ is exactly realizable in dimension $d$, meaning that there is a set of points $y_1, \dots, y_m$ such that $\Delta_{ij} = \|y_i - y_j\|^2$.
It is worth noting that, in that case, the set of points that perfectly embed $\Delta$ in dimension $d$ are rigid transformations of each other.
It is well-known that classical scaling provides such a set of points which happens to be centered at the origin (see Eq.~\ref{eq:MDSidentity}).

We perform a perturbation analysis of classical scaling, by studying the effect of perturbing the dissimilarities on the embedding that the algorithm returns.  This sort of analysis helps quantify the degree of robustness of a method to noise, and is particularly important in applications where the dissimilarities are observed with some degree of inaccuracy, which is the case in the context of manifold learning (\secref{isomap}).

\ajr{
\begin{Def}
We say that $\Delta\in \bbR^{m\times m}$ is a $d$-Euclidean dissimilarity matrix if there exists a set of points $\{x_1,\dotsc, x_m\}\in \bbR^d$ such that $\Delta_{ij} = \|x_i - x_j\|^2$.  
\end{Def}
}

Recall that $\cO$ denotes the orthogonal group of matrices in the appropriate Euclidean space (which will be clear from context).

\begin{cor}
\label{thm:classical scaling}
\ajR{Let $\Lambda, \Delta \in \bbR^{m\times m}$ denote two $d$-Euclidean dissimilarity matrices, with $\Delta$ corresponding to a centered and full rank configuration $Y\in \reals^{m\times d}$.}
Set $\eps^2 = \frac12 \|H(\Lambda - \Delta)H\|_p$.
If it holds that \ajR{$\|Y^\ddag\| \eps \le \frac1{\sqrt{2}}$}, 
 then classical scaling with input dissimilarity matrix $\Lambda$ and dimension $d$ returns a centered configuration $Z \in \bbR^{m \times d}$ satisfying 
\ajR{
\beq \label{mds}
\min_{Q \in \cO} \|Z - Y Q \|_p \le (1 + \sqrt{2}) \|Y^\ddag\| \eps^2.
\eeq
}
\end{cor}

We note that $\eps^2 \le \tfrac12 d^{2/p} \|\Lambda - \Delta\|_p$, after using the fact that $\|H\|_p = (d-1)^{1/p}$ since $H$ has one zero eigenvalue and $d-1$ eigenvalues equal to one.

\begin{proof}
We have
\beq
\|\Lambda^c - \Delta^c\|_p = \tfrac12 \|H (\Lambda - \Delta) H\|_p = \eps^2.
\eeq
\ajr{Note that since $\Delta$ and $\Lambda$ are both $d$-Euclidean dissimilarity matrices, using identity~\eqref{eq:MDSidentity}, the doubly centered matrices $\Delta^c$ and $\Lambda^c$ are both positive semi-definite and of rank at most $d$. Indeed, since $Y$ is full rank (rank $d$) and centered, then \eqref{eq:MDSidentity} implies that $\Delta^c$ is of rank $d$.	
Therefore, for the underlying configuration $Y$ and the configuration $Z$, returned by classical scaling, we have} $\Delta^c = Y Y^\top$ and $\Lambda^c = Z Z^\top$. We next simply apply \thmref{procrustes}, which we can do since $Y$ has full rank by assumption, to conclude.
\end{proof}

\begin{rem}
The perturbation bound \eqref{mds} is optimal in how it depends on $\eps$.  Indeed, suppose without loss of generality that $p = 2$.  (All the Schatten norms are equivalent modulo constants that depend on $d$ and $p$.)  Consider a configuration $Y$ \ajr{with squared distance matrix $\Delta$ as in the statement, and define $\Lambda = (1+a)^2 \Delta$, with $0 \le a \le 1$, as a perturbation of $\Delta$. Then, it is easy to see that classical scaling with input dissimilarity matrix $\Lambda$ returns $Z = (1+a) Y$.}  
On the one hand, we have \cite[Sec 5.6]{seber2009multivariate}
\beq
\min_{Q \in \cO} \|Z - Y Q \|_2 = \|Z - Y\|_2 = a \|Y\|_2\,.
\eeq   
On the other hand,
\ajr{
\begin{align}
\eps^2 = \frac12 \|H(\Lambda - \Delta)H\|_p
 = \frac12 ((1+a)^2 - 1) \|H\Delta H\|_p
 = ((1+a)^2 - 1) \|YY^\top\|_2\,.
\end{align}
}
Therefore, the right-hand side in \eqref{mds} can be bounded by \ajR{$3(1+\sqrt{2}) a \|Y^\ddag\| \|YY^\top\|_2$}, 
 using that $a \in [0,1]$.
We therefore conclude that the ratio of the left-hand side to the right-hand side in \eqref{mds} is at least 
\ajR{
\beq
\frac{a \|Y\|_2}{3(1+\sqrt{2}) a \|Y^\ddag\| \|YY^\top\|_2}
\ge \frac{1}{3(1+\sqrt{2})} (\|Y\| \|Y^\ddag\|)^{-1},
\eeq
}
using the fact that $\|YY^\top\|_2 \le \|Y\| \|Y\|_2$.
Therefore, our bound \eqref{mds} is tight up to a multiplicative factor depending on the condition number of the configuration $Y$.
\end{rem}

\ajR{\begin{rem}
Condition $\|Y^\ddag\| \eps\le \tfrac{1}{\sqrt{2}}$ in Corollary~\ref{thm:classical scaling} is of crucial importance in that without it the dissimilarity matrix $\Lambda$ may have rank less than $d$. In this case, the classical scaling (Algorithm~\ref{alg:classical scaling}) with input $\Lambda$, returns a configuration $Z$ which contains zero columns and hence suffers a large procrustes error.
\end{rem}}
\paragraph{}
We now translate this result in terms of point sets instead of matrices.  For a centered point set $y_1, \dots, y_m \in \bbR^d$, stored in the matrix $Y = [y_1 \cdots y_m]^\top \in \bbR^{m \times d}$, define its radius as the largest standard deviation along any direction in space (therefore corresponding to the square root of the top eigenvalue of the covariance matrix).  We denote this by $\rho(Y)$ and note that
\beq\label{radius}
\rho(Y) = \|Y\|/\sqrt{m}.
\eeq
We define its half-width as the smallest standard deviation along any direction in space (therefore corresponding to the square root of the bottom eigenvalue of the covariance matrix).  We denote this by $\omega(Y)$ and note that it is strictly positive if and only if the point set $\{y_1,\dotsc, y_m\}$ \ajr{spans the whole space $\bbR^d$; in other words, the matrix $Y = [y_1 \cdots y_m]^\top \in \bbR^{m \times d}$ is of rank $d$}. In this case 
\beq\label{with}
\omega(Y) = \|Y^\ddag\|^{-1}/\sqrt{m}.
\eeq
It is well-known that the half-width quantifies the best affine approximation to the point set, in the sense that
\beq \label{hyperplane}
\omega(Y)^2 = \min_{\cL} \frac1m \sum_{i \in [m]} \|y_i - P_\cL y_i\|^2,
\eeq
where the minimum is over all affine hyperplanes $\cL$, and for a subspace $\cL$, $P_\cL$ denotes the orthogonal projection onto $\cL$.
We note that $\rho(Y)/\omega(Y) = \|Y\| \|Y^\ddag\|$ is the aspect ratio of the point set.

\begin{cor}
\label{cor:classical scaling}
Consider a centered point set $y_1, \dots, y_m \in \bbR^d$ with radius $\rho$, and with half-width $\omega$, and with pairwise dissimilarities $\delta_{ij} = \|y_i - y_j\|^2$.  
Consider another \ajr{arbitrary set of numbers $\{\lambda_{ij}\}$, for $1\le i, j\le m$} and set \smash{$\eta^4 = \frac1{m^2} \sum_{i,j} (\lambda_{ij} - \delta_{ij})^2$}.
If \ajR{$\eta/\omega \le \tfrac{1}{\sqrt{2}}$},
then classical scaling with input dissimilarities $\{\lambda_{ij}\}$ and dimension $d$ returns a point set $z_1 \cdots z_m \in \bbR^d$ satisfying
\beq \label{mds1}
\min_{Q \in \cO} \bigg(\frac1m \sum_{i \in [m]} \|z_i - Q y_i\|^2\bigg)^{1/2} \le \frac{\sqrt{d} (\rho/\omega + 2)}{\omega}\, \eta^2 \le \frac{3\sqrt{d} \rho\, \eta^2}{\omega^2}.
\eeq
\end{cor}
This corollary follows from \thmref{procrustes}. We refer to Section~\ref{proof:cor:classical scaling} for its proof. 

\ajr{
\begin{rem}
In some applications, one might be interested in an approximate embedding, where the goal is to embed a large fraction (but not necessarily all) of the points with high accuracy. Note that bound~\eqref{mds1} provides a non-trivial bound for this objective. Indeed, for any optimal $Q$ for the left-hand side of \eqref{mds1}, and an arbitrary fixed $\delta >0$, let $N(\delta)\equiv |\{i\in [m]:\, \|z_i - Qy_i\| > \delta\}|$ be the number of points that are not embedded within accuracy $\delta$. Then, \eqref{mds1} implies that
\begin{align*}
N(\delta) \delta^2 \le \sum_{i\in [m]} \|z_i - Qy_i\|^2\le m \left( \frac{3\sqrt{d} \rho\, \eta^2}{\omega^2}\right)^2\,,
\end{align*}
and hence
\beq
N(\delta)\le m \left( \frac{3\sqrt{d} \rho\, \eta^2}{\delta\omega^2}\right)^2\,.
\eeq
\end{rem}
}
\section{A perturbation bound for trilateration}
\label{sec:trilateration}
The problem of trilateration is that of positioning a point, or set of points, based on its (or their) distances to a set of points, which in this context serve as landmarks.  In detail, given a set of landmark points $y_1, \dots, y_m \in \bbR^d$ and a set of dissimilarities $\tdelta_1, \dots, \tdelta_m$, the goal is to find $\ty \in \bbR^d$ such that $\|\ty - y_i\|^2$ is close to $\tdelta_i$ over all $i \in [m]$.
\algref{trilateration} describes the trilateration method of \citet{de2004sparse} simultaneously applied to multiple points to be located.  
The procedure is shown in \cite{de2004sparse} to \ajr{recover the position of points $\ty_1,\dotsc, \ty_n$ exactly, when it is given the squared distances $\tdelta_{ij} = \|\ty_i - y_j\|^2$ as input and the landmark point set $\{y_1,\dotsc, y_m\}$ spans $\bbR^d$.} We provide a more succinct proof of this in the \secref{app_trilateration}. 

\begin{algorithm}[!ht]
\caption{Trilateration}
\label{alg:trilateration}
\begin{algorithmic}
\STATE {\bf Input:} centered point set $y_1, \dots, y_m \in \bbR^d$, dissimilarities $\tDelta = (\tdelta_{ij}) \in \bbR^{n \times m}$
\STATE {\bf Output:} points $\ty_1, \dots, \ty_n \in \bbR^d$
\medskip
\STATE {\bf 1:} compute $\bar a = \frac1m \sum_{i=1}^m a_i$, where $a_i = (\|\ty_i - y_1\|^2, \dots, \|\ty_i - y_m\|^2)$
\STATE {\bf 2:} compute the pseudo-inverse $Y^\ddag$ of $Y = [y_1 \cdots y_m]^\top$
\STATE {\bf 3:} compute $\tY^\top = \frac12 Y^\ddag (\bar a \one^\top - \Delta^\top)$
\STATE {\bf Return:} the row vectors of $\tY$, denoted $\ty_1, \dots, \ty_n \in \bbR^d$
\end{algorithmic}
\end{algorithm}

We perturb both the dissimilarities and the landmark points, and qualitatively characterize how it will affect the returned positions by trilateration.
(In principle, the perturbed point set need not have the same mean as the original point set, but we assume this is the case, for simplicity and because it suffices for our application of this result in \secref{applications}.)
For a configuration $Y = [y_1 \cdots y_m]^\top$, define its max-radius as
\beq
\rho_\infty(Y) = \max_{i \in [m]} \|y_i\|,
\eeq
and note that $\rho(Y) \le \rho_\infty(Y)$.
We content ourselves with a bound in Frobenius norm.\footnote{ All Schatten norms are equivalent here up to a multiplicative constant that depends on $d$, since the matrices that we consider have rank of order $d$.}

\begin{thm}
\label{thm:trilateration}
Consider a centered configuration $Y \in \bbR^{m \times d}$ that spans the whole space \ajr{$\bbR^d$}, and for a given configuration $\tY \in \bbR^{n \times d}$, let $\tDelta \in \bbR^{n \times m}$ denote the matrix of dissimilarities between $\tY$ and $Y$\ajr{, namely $\tDelta_{ij} = \|\ty_i-y_j\|^2$}. 
Let $Z \in \bbR^{m \times d}$ be another centered configuration that spans the whole space, and let $\tilde{\Lambda} \in \bbR^{n \times m}$ be \ajr{an arbitrary matrix}. 
Then, trilateration with inputs $Z$ and $\tilde \Lambda$ returns $\tilde Z \in \bbR^{n \times d}$ satisfying 
\begin{multline} \label{trilateration}
\|\tilde Z - \tY\|_2 
\le \tfrac12 \|Z^\ddag\| \|\tilde \Lambda - \tDelta\|_2 
+ 2 \|\tY\| \|Z^\ddag\| \|Z - Y\|_2 \\
+ \ajr{3} \sqrt{m} (\rho_\infty(Y) + \rho_\infty(Z)) \|Z^\ddag\| \|Z - Y\|_2 
+ \|Y\| \|\tY\| \|Z^\ddag - Y^\ddag\|_2\,.
\end{multline}
\end{thm}

\ajr{In the bound~\eqref{trilateration}, we see that the first term captures the effect of the error in the dissimilar matrix, i.e., $\|\tDelta - \tilde\Lambda\|$, while the other three terms reflect the impact of the error in the landmark positions, i.e, $\|Z-Y\|$. As we expect, we have a more accurate embedding as these two terms get smaller, and in particular, when $\tDelta = \tilde\Lambda$ and $Y = Z$ (no error in the inputs), we have exact recovery, which corroborates our derivation in \secref{app_trilateration}. }

\begin{rem}
\label{rem:trilateration}
\ajr{For a bound not involving the pseudo-inverse of $Z$ -- which may be  difficult to interpret -- we can upper bound the right-hand side of~\eqref{trilateration} using}
\beq
\rho_\infty(Z) \le \rho_\infty(Y) + \rho_\infty(Z - Y), \quad  
\|Z^\ddag\| \le \|Y^\ddag\| + \|Z^\ddag - Y^\ddag\|, 
\eeq
and
\beq\label{trilateration-aux1}
\|Z^\ddag - Y^\ddag\|_p \le \frac{\sqrt{2} \|Y^\ddag\|^2 \|Z - Y\|_p}{(1 - \|Y^\ddag\| \|Z - Y\|)_+^2}, \quad p \in \{2, \infty\},
\eeq
as per Lemma~\ref{lem:wedin}.
\new{Also, a simple application of Mirsky's inequality \eqref{mirsky} implies that, when $Y$ spans the whole space then so does $Z$ whenever $\|Y^\ddag\| \|Z - Y\| < 1$.}
\end{rem}

The proof is in \secref{proof_trilateration}.
We now derive from this result another one in terms of point sets instead of matrices.
\begin{cor}
\label{cor:trilateration}
Consider a centered point set $y_1, \dots, y_m \in \bbR^d$ with radius $\rho$, max-radius $\rho_\infty$, and half-width $\omega > 0$.  
For a point set $\ty_1, \dots, \ty_n \in \bbR^d$ with radius $\zeta$, set $\tdelta_{ij} = \|\ty_i - y_j\|^2$.
Also, let $z_1, \dots, z_m \in \bbR^d$ denote another centered point set, and let $(\tlambda_{ij})$ denote another \ajr{arbitrary set of numbers for $1\le i\le n$, $1\le j\le m$}.  Set $\eps = \max_{i \in [m]} \|z_i - y_i\|$ and $\eta^4 = \frac1{nm} \sum_{ij} (\tlambda_{ij} - \tdelta_{ij})^2$.
If $\eps \le \omega/{2}$, trilateration with inputs $z_1, \dots, z_m$ and $(\tlambda_{ij})$ returns $\tilde z_1, \dots, \tilde z_n \in \bbR^d$ satisfying 
\beq\label{mybound}
\bigg(\frac1n \sum_{i \in [n]} \|\tilde z_i - \ty_i\|^2\bigg)^{1/2} 
\le C_0 \left( \frac{\eta^2}{\omega} + \left[\frac{\rho\zeta}{\omega^2} + \frac{\sqrt{m} \rho_\infty}{\sqrt{n}\omega}\right] \eps \right)\,,
\eeq
where $C_0$ is a universal constant.
\end{cor}

Corollary~\ref{cor:trilateration} follows from \thmref{trilateration} and its proof is given in Section~\ref{proof:cor:trilateration}. 

\ajr{
\begin{rem}\label{rem:justification1}
In the bound~\eqref{mybound}, the terms $\eps$ and $\eta$ respectively quantify the errors in the positions of landmarks and the error in the dissimilarities that are fed to the trilateration procedure. 
As we expect, smaller values of $\eps$ and $\eta$ lead to a more accurate embedding of the points, and in the extreme situation where $\eps = 0$ and $\eta = 0$, we can infer the positions of the points $\ty_i$ exactly.
Also, the bound is reciprocal in the half-width of the landmark set, $\omega$. This is also expected because a small $\omega$ means that the landmarks have small dispersion along some direction in $\bbR^d$ and hence the positions of other points cannot be well approximated along that direction. This can also be seen from Step 3 of the trilateration procedure. The quantity $\|Y^\ddagger\|$ measures the sensitivity of $X$ to the dissimilarities $\Delta$. Invoking~\eqref{with}, $\omega = \|Y^\ddagger\|^{-1}/m$, and hence a small $\omega$ corresponds to large sensitivity, meaning that a small perturbation in $\Delta$ can lead to large errors in $X$. This is consistent with our bound as the error in $\Delta$, i.e., $\eta^2$ appears by the scaling factor $1/\omega$. 
\end{rem}
}

\section{Applications to manifold learning}
\label{sec:applications}
Consider a set of points in a possibly high-dimensional Euclidean space, that lie on a smooth Riemannian manifold.
Isometric manifold learning (or embedding) is the problem of embedding these points into a lower-dimensional Euclidean space, and do as while preserving as much as possible the Riemannian metric.  
There are several variants of the problem under other names, such as nonlinear dimensionality reduction.  

\begin{rem}
Manifold learning is intimately related to the problem of embedding items with only partial dissimilarity information, which practically speaking means that some of the dissimilarities are missing.  We refer to this problem as graph embedding below, although it is known under different names such as graph realization, graph drawing, and sensor localization.  This connection is due to the fact that, in manifold learning, the short distances are nearly Euclidean, while the long distances are typically not.  In fact, the two methods for manifold learning that we consider below can also be used for graph embedding.  The first one, Isomap \cite{Tenenbaum00ISOmap}, coincides with MDS-MAP  \cite{shang2003localization} (see also \cite{niculescu2003dv}), although the same method was suggested much earlier by \citet{kruskal1980designing}; the second one, Maximum Variance Unfolding, was proposed as a method for graph embedding by  the same authors \cite{weinberger2006graph}, and is closely related to other graph embedding methods \citep{biswas2006semidefinite, javanmard2013localization, so2007theory}.
\end{rem}

\subsection{A performance bound for (Landmark) Isomap}
\label{sec:isomap}
Isomap is a well-known method for manifold learning, suggested by \citet*{Tenenbaum00ISOmap}.  
\algref{isomap} describes the method.  
(There, we use the notation $A^{\circ2}$ to denote the matrix with entries $A_{ij}^2$.)

\begin{algorithm}[!t]
\caption{Isomap}
\label{alg:isomap}
\begin{algorithmic}
\STATE {\bf Input:} data points $x_1, \dots, x_n \in \bbR^D$, embedding dimension $d$, neighborhood radius $r$
\STATE {\bf Output:} embedding points $z_1, \dots, z_n \in \bbR^d$
\medskip
\STATE {\bf 1:} construct the graph on $[n]$ with edge weights $w_{ij} = \|x_i - x_j\|\, \IND{\|x_i - x_j\| \le r}$
\STATE {\bf 2:} compute the shortest-path distances in that graph $\Gamma = (\gamma_{ij})$
\STATE {\bf 3:} apply classical scaling with inputs $\Gamma^{\circ 2}$ and $d$, resulting in points $z_1, \dots, z_n \in \bbR^d$
\STATE {\bf Return:} the points $z_1, \dots, z_n$
\end{algorithmic}
\end{algorithm}

There are two main components to Isomap: 1) Form the $r$-ball neighborhood graph based on the data points and compute the shortest-path distances; 2) Pass the obtained distance matrix to classical scaling (together with the desired embedding dimension) to obtain an embedding.
The algorithm is known to work well when the underlying manifold is isometric to a convex domain in $\bbR^d$.  Indeed, assuming an infinite sample size, so that the data points are in fact all the points of the manifold, as $r \to 0$, the shortest-path distances will converge to the geodesic distances on the manifold, and thus, in that asymptote (infinite sample size and infinitesimal radius), an isometric embedding in $\bbR^d$ is possible under the stated condition.
We will assume that this condition, that the manifold is isometric to a convex subset of $\bbR^d$, holds.

In an effort to understand the performance of Isomap, \citet{bernstein2000graph} study how well the shortest-path distances in the $r$-ball neighborhood graph approximate the actual geodesic distances.  
Before stating their result we need to state a definition.
\begin{Def}
The \emph{reach} of a subset $\cA$ in some Euclidean space is the supremum over $t \ge 0$ such that, for any point $x$ at distance at most $t$ from $\cA$, there is a unique point among those belonging to $\cA$ that is closest to $x$.  When $\cA$ is a $C^2$ submanifold, its reach is known to bound its radius of curvature from below \cite{MR0110078}.
\end{Def}
Assume that the manifold $\M$ has reach at least $\tau > 0$, and the data points are sufficiently dense in that 
\beq\label{a}
\min_{i \in [n]} g_\M(x, x_i) \le a, \quad \forall x \in \M,
\eeq
where $g_\M$ denote the metric on $\M$ (induced by the surrounding Euclidean metric). If $r$ is sufficiently small in that $r < \tau$, then \citet{bernstein2000graph} show that
\beq
1 - c_0 (r/\tau)^2 \le \frac{\gamma_{ij}}{g_{ij}} \le 1 + c_0 (a/r), \quad \forall i, j \in [n],
\eeq
where $\gamma_{ij}$ is the graph distance, $g_{ij}$ is the geodesic distance between $x_i$ and $x_j$, and $c_0 \ge 1$ is a universal constant.
\new{(In fact, \citet{bernstein2000graph} derive such a bound under the additional condition that $\M$ is geodesically convex, although the result can be generalized without much effort \cite{arias2017unconstrained}.)}

We are able to improve the upper bound in the restricted setting considered here, where the underlying manifold is assumed to be isometric to a convex domain.

\begin{prp}
\label{prp:better BSLT}
In the present situation, there is a universal constant $c_1 \ge 1$ such that, if $a/r \le 1/\sqrt{c_1}$, 
\beq
\frac{\gamma_{ij}}{g_{ij}} \le 1 + c_1 (a/r)^2, \quad \forall i, j \in [n].
\eeq
\end{prp}

\ajr{Thus, if we set 
\begin{align}\label{eq:xi}
\xi = c_0 (r/\tau)^2 \vee c_1 (a/r)^2\,,
\end{align}
using the notation $a\vee b = \max(a,b)$,} and it happens that $\xi < 1$, we have
\beq\label{BSLT}
1 - \xi \le \frac{\gamma_{ij}}{g_{ij}} \le 1 + \xi, \quad \forall i, j \in [n].
\eeq

Armed with our  perturbation bound for classical scaling, we are able to complete the analysis of Isomap, obtaining the following performance bound.

\begin{cor}
\label{cor:isomap}
In the present context, let $y_1, \dots, y_n \in \bbR^d$ denote a possible (exact and centered) embedding of the data points $x_1, \dots, x_n \in \M$, and let $\rho$ and $\omega$ denote the max-radius and half-width of the embedded points, respectively.  
\ajr{Let $\xi$ be defined by Equation~\eqref{eq:xi}.} If \ajR{$\xi \le \frac1{24}(\rho/\omega)^{-2}$}, then Isomap returns $z_1, \dots, z_n \in \bbR^d$ satisfying
\beq\label{eq:isomap-bound}
\min_{Q \in \cO} \bigg(\frac1n \sum_{i \in [n]} \|z_i - Q y_i\|^2\bigg)^{1/2} 
\le \frac{36 \sqrt{d} \rho^3}{\omega^2} \xi.
\eeq
\end{cor}

\ajr{
\begin{rem}\label{rem:justification2}
As we can see, the performance of Isomap degrades as $\omega$ gets smaller, which we already justified in Remark~\ref{rem:justification1}. Also the performance improves for smaller values of $\xi$. Recalling the definition of $\xi$ in \eqref{eq:xi}, fixing $r$, a smaller $\xi$ corresponds to a denser set of points on the manifold, i.e., a smaller $a$, and also a smaller reach, i.e., a smaller $\tau$, which leads the graph distances to better approximate the geodesic distances.
\end{rem}}

\begin{figure}[t!]
	\centering
    \subfloat[Data points $x_i \in \M$ and the $r$-ball neighborhood graph ]{\includegraphics[scale = 0.5]{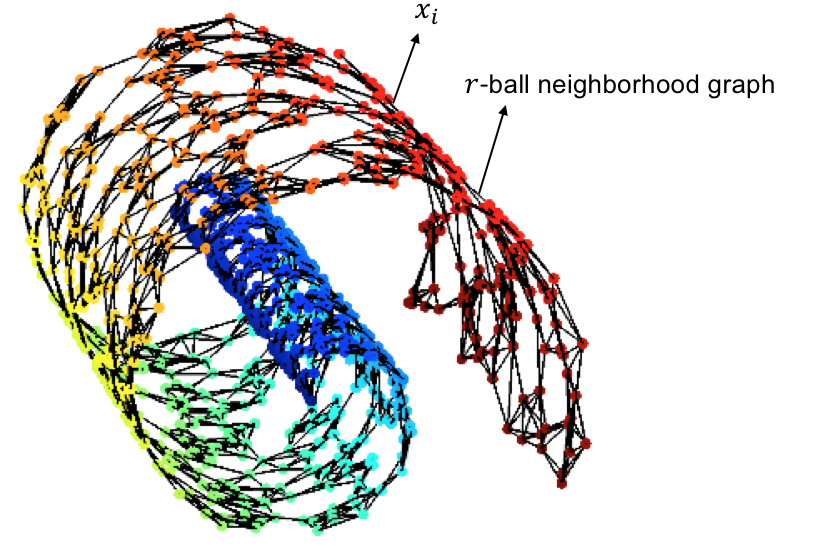}}\\
    \subfloat[Exact embedding onto the low-dimensional space ]{\includegraphics[scale=0.45]{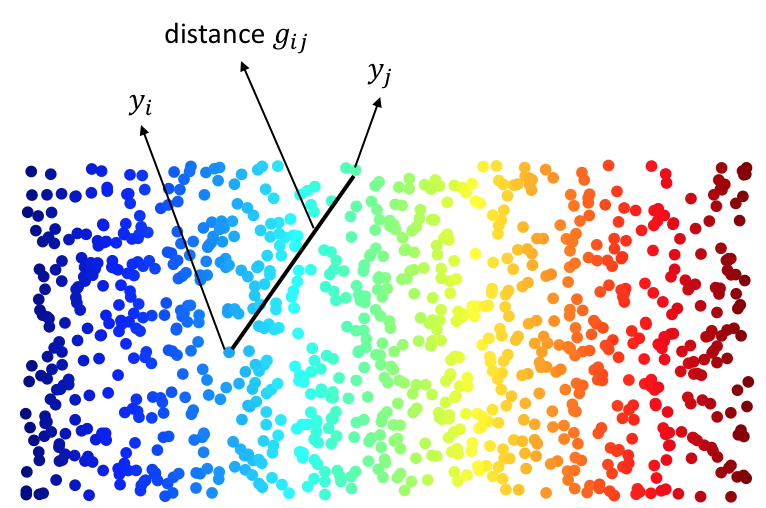}}~\hspace{1.6cm}
    \subfloat[Returned locations by Isomap]{\includegraphics[scale=0.45]{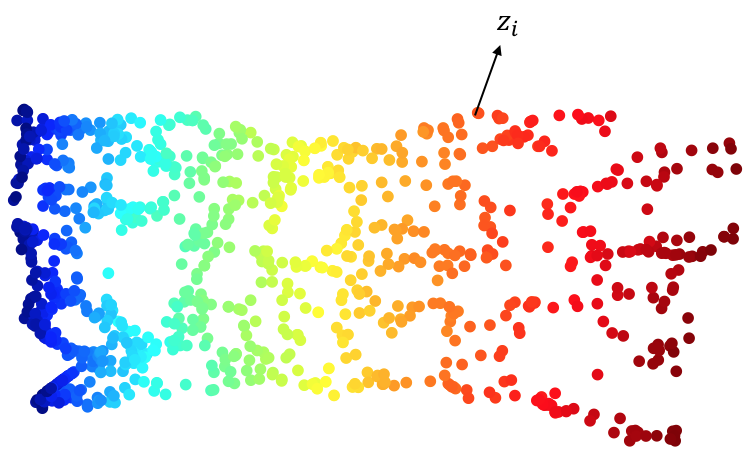}}
	\caption{Schematic representation of exact locations $y_i\in \bbR^d$, data points $x_i\in \M$, returned locations by Isomap $z_i\in \bbR^d$. Note that $g_{ij} = \|y_i-y_j\|$ is the geodesic distance between $x_i$ and $x_j$ because
	$\{y_i\}_{i=1}^n$ is an exact isometric embedding of data points $\{x_i\}_{i=1}^n$. Also the distances $\gamma_{ij}$ are computed as shortest path distances between $x_i$ and $x_j$ on the $r$-ball neighborhood graph.}
	\label{fig:schematic}
\end{figure}

\begin{proof}
Before we provide the proof, we refer to \figref{schematic} for a schematic representation of exact locations $y_i\in \bbR^d$, data points $x_i\in \M$, returned locations by Isomap $z_i\in \bbR^d$, as well as graph distances $\gamma_{ij}$ and geodesic distances $g_{ij}$.

The proof itself is a simple consequence of \corref{classical scaling}.  Indeed, with \eqref{BSLT} it is straightforward to obtain (with $\eta$ defined in \corref{classical scaling} and $\gamma_{ij}$ and $g_{ij}$ as above), 
\beq\label{eta}
\eta^2 
\le \max_{i, j \in [n]} |\gamma_{ij}^2 - g_{ij}^2| 
\le \max_{i, j \in [n]} (2\xi + \xi^2) g_{ij}^2
\le (2\xi + \xi^2) (2\rho)^2
\le 12 \rho^2 \xi,
\eeq
\ajr{where in the last step we used the fact that $\xi < 1$ because \ajR{$\xi \le \frac1{24}(\rho/\omega)^{-2}$} by our assumption and $\omega \le \rho$, by definition.} 
In particular, $\eta$ fulfills the conditions of \corref{classical scaling} under the stated bound $\xi$, so we may conclude by applying that corollary and simplifying. 
\end{proof}

If $\cD$ is a domain in $\bbR^d$ that is isometric to $\M$, then the radius of the embedded points ($\rho$ above) can be bounded from above by the radius of $\cD$, and under mild assumptions on the sampling, the half-width of the embedded points ($\omega$ above) can be bounded from below by a constant times the half-width of $\cD$, in which case $\rho$ and $\omega$ should be regarded as fixed.
Similarly, $\tau$ should be considered as fixed, so that the bound is of order $O(r^2 \vee (a/r)^2)$, optimized at $r \asymp a^{1/2}$.
If the points are well spread-out, for example if the points are sampled iid from the uniform distribution on the manifold, then $a$ is on the order of $(\log(n)/n)^{1/d}$, and the bound (with optimal choice of radius) is $O((\log(n)/n)^{1/d})$.


\paragraph{Landmark Isomap}
Because of the relatively high computational complexity of Isomap, and also of classical scaling, \citet{de2004sparse, silva2003global} proposed a Nystr\"om approximation (as explained in \cite{platt2005fastmap}).  Seen as a method for MDS, it starts by embedding a small number of items, which effectively play the role of landmarks, and then embedding the remaining items by trilateration based on these landmarks.
Seen as a method for manifold learning, the items are the points in space, and the dissimilarities are the squared graph distances, which are not provided and need to be computed.
\algref{landmark isomap} details the method in this context.  
The landmarks may be chosen at random from the data points, although other options are available, and we discuss some of them in \secref{app_landmarks}. 

\begin{algorithm}[!ht]
\caption{Landmark Isomap}
\label{alg:landmark isomap}
\begin{algorithmic}
\STATE {\bf Input:} data points $x_1, \dots, x_n \in \bbR^D$, embedding dimension $d$, neighborhood radius $r$, number of landmarks $\ell$
\STATE {\bf Output:} embedding points $\{z_i:\, i\in \cL\}\cup\{\tz_i:\, i\notin \cL\} \subseteq \bbR^d$ for a choice of $|\cL|=\ell$ landmarks
\medskip
\STATE {\bf 1:} construct the graph on $[n]$ with edge weights $w_{ij} = \|x_i - x_j\|\, \IND{\|x_i - x_j\| \le r}$
\STATE {\bf 2:} select $\cL \subset [n]$ of size $\ell$ according to one of the methods in \secref{app_landmarks}
\STATE {\bf 3:} compute the shortest-path distances in that graph $\Gamma = (\gamma_{ij})$ for $(i,j) \in [n] \times \cL$
\STATE {\bf 4:} apply classical scaling with inputs $\Gamma^{\circ 2}_{\cL \times \cL}$ and $d$, resulting in (landmark) points $z_i, i \in \cL$ in $\bbR^d$
\STATE {\bf 5:} for each $i \notin \cL$, apply trilateration based on $\{z_j : j \in \cL\}$ and $\Gamma^{\circ 2}_{i \times \cL}$ to obtaining $\tz_i \in \bbR^d$
\STATE {\bf Return:} the points $\{z_i:\, i\in \cL\}\cup\{\tz_i:\, i\notin \cL\}$.
\end{algorithmic}
\end{algorithm}

With our work, we are able to provide a performance bound for Landmark Isomap.

\begin{cor}
\label{cor:landmark isomap}
\ajR{
Consider $n$ data points $x_1, \dots, x_n \in \M$, which has a possible (exact and centered) embedding in $\reals^d$.
Let $\cL$ be a subset of the points ($|\cL| = \ell$) with exact embedding $\{y_1,\dotsc, y_\ell\}$ and denote the embedding of the other points by $\ty_1,\dotsc, \ty_{n-\ell}$. Assume that $\{y_1,\dotsc,y_m\}$ has half-width $\omega_* > 0$, the exact embedding $\{y_1,\dotsc, y_\ell\}\cup\{\ty_1,\dotsc,\ty_{n-\ell}\}$ has maximum-radius $\rho$, and
{$\ell \le (n/2)\wedge [(72\sqrt{d}\xi)^{-2} (\rho/\omega_*)^{-6}]$}. 
Then the Landmark Isomap, with the choice of $\cL$ as landmarks, returns $\{z_1,\dotsc,z_\ell\}\cup\{\tz_1,\dotsc,\tz_{n-\ell}\} \subseteq \bbR^d$  satisfying
\beq\label{eq:landmark-isomap}
\min_{Q \in \cO} \bigg(\frac1n \sum_{i\in [\ell]}  \|z_i - Q y_i\|^2 + \frac1n \sum_{i \in [n-\ell]} \|\tz_i - Q \ty_i\|^2\bigg)^{1/2} 
\le {C_1 \frac{\rho^2}{\omega_*}}\,,
\eeq
where $C_1$ is a universal constant.}
\end{cor}

The result is a direct consequence of applying \corref{isomap}, which allows us to control the accuracy of embedding the landmarks using classical scaling, followed by applying \corref{trilateration}, which allows us to control the accuracy of embedding using trilateration. 
The proof is given in Section~\ref{proof:cor:landmark isomap}. \ajr{As we see our bound~\eqref{eq:landmark-isomap} on the embedding error improves when the half-width of the landmarks, $\omega_*$, increases. We justified this observation in~\remref{justification1}: a higher half-width of the landmarks yields a better performance of the trilateration procedure. In \secref{app_landmarks}, we use this observation to provide guidelines for choosing landmarks.}  

We note that for the set of (embedded) landmarks to have positive half-width, it is necessary that they span the whole space, which compels $\ell \ge d+1$.  In \secref{app_landmarks} we show that choosing the landmarks at random performs reasonably well in that, with probability approaching 1 very quickly as $\ell$ increases, their (embedded) half-width is at least half that of the entire (embedded) point set.

\subsection{A performance bound for Maximum Variance Unfolding}
\label{sec:mvu}

Maximum Variance Unfolding is another well-known method for manifold learning, proposed by
\citet{weinberger2006unsupervised,weinberger2006introduction}.  \algref{mvu} describes the method, which relies on solving a semidefinite relaxation.
\ajr{There is also an interpretation of MVU as a regularized shortest path solution~\cite[Theorem 2]{paprotny2012connection}.}

\begin{algorithm}[!ht]
\caption{ Variance Unfolding (MVU)}
\label{alg:mvu}
\begin{algorithmic}
\STATE {\bf Input:} data points $x_1, \dots, x_n \in \bbR^D$, embedding dimension $d$, neighborhood radius $r$
\STATE {\bf Output:} \ajr{embedded points $z_1, \dots, z_n \in \bbR^d$}
\medskip
\STATE{\bf 1:} set $\gamma_{ij} = \|x_i - x_j\|$ if $\|x_i - x_j\| \le r$, and $\gamma_{ij} = \infty$ otherwise
\STATE {\bf 2:} solve the following semidefinite program
\ajr{
\beqn
\text{maximize } \sum_{i, j \in [n]} \|p_i - p_j\|^2 \text{ over } p_1, \dots, p_n \in \bbR^D, \text{ subject to } \|p_i - p_j\| \le \gamma_{ij}
\eeq

\STATE {\bf 3:} center a solution set and embed it into $\bbR^d$ using principal component analysis}
\STATE {\bf Return:} the embedded point set, denoted by \ajr{$z_1, \dots, z_n$}
\end{algorithmic}
\end{algorithm}

Although MVU is broadly regarded to be more stable than Isomap, \citet{arias2013convergence} show that it works as intended under the same conditions required by Isomap, namely, that the underlying manifold is geodesically convex.  
Under these conditions, in fact, under the same conditions as in \corref{isomap}, where in particular \eqref{BSLT} is assumed to hold with $\xi$ sufficiently small, \citet{paprotny2012connection} show that MVU returns an embedding, $z_1, \dots, z_n \in \bbR^d$, with dissimilarity matrix \ajr{$\Lambda = (\lambda_{ij})$, $\lambda_{ij} = \|z_i - z_j\|^2$}, satisfying
\beq\label{paprotny}
|\Lambda - \Delta|_1 \le 9 \rho^2 n^2 \xi,
\eeq
where $\Delta = (\delta_{ij})$, $\delta_{ij} = \|y_i - y_j\|^2$ (the correct underlying distances), and for a matrix $A = (a_{ij})$, $|A|_p^p = \sum_{i,j} |a_{ij}|^p$.
Based on that, and on our work in \secref{classical scaling}, we are able to provide the following performance bound for MVU, which is \ajr{similar} to the bound we obtained for Isomap.

\begin{cor}
\label{cor:mvu}
\ajR{Let $y_1, \dots, y_n \in \bbR^d$ denote a possible (exact and centered) embedding of the data points $x_1, \dots, x_n \in \M$, and let $\rho$ and $\omega$ denote the max-radius and half-width of the embedded points, respectively.  Suppose that the neighborhood radius $r$ is chosen so that the corresponding neighborhood graph on points $\{x_i\}_{i\in [n]}$ is connected. 
Let $\xi$ be defined by Equation~\eqref{eq:xi}. If $\xi \le (12 \sqrt{3})^{-1} (\rho/\omega)^{-2}$}, then Maximum Variance Unfolding returns $z_1, \dots, z_n \in \bbR^d$ satisfying
\beq
\min_{Q \in \cO} \bigg(\frac1n \sum_{i \in [n]} \|z_i - Q y_i\|^2\bigg)^{1/2} 
\le \frac{18 \sqrt{3d} \rho^3}{\omega^2} \xi.
\eeq
\end{cor}

\begin{proof}
As in \eqref{eta}, we have 
\beq
|\Lambda - \Delta|_\infty
= \max_{i, j \in [n]} |\gamma_{ij}^2 - g_{ij}^2
|\le 12 \rho^2 \xi,
\eeq
so that, in combination with \eqref{paprotny}, we have 
\beq
\|\Lambda - \Delta\|_2
\le |\Lambda - \Delta|_\infty^{1/2} |\Lambda - \Delta|_1^{1/2}
\le 6\sqrt{3} n \rho^2 \xi.
\eeq
In particular, the conditions of \corref{classical scaling} are met under the stated bound on $\xi$.  Therefore, we may apply that corollary to conclude.
\end{proof}

\ajR{\section{Numerical Experiments}\label{sec:numerical}
\paragraph{Procrustes problem.} 
We let $n = 100$, $d = 10$ and generate $X\in \bbR^{n\times d}$ as $X = UDV^\top$, where $U,V\in \bbR^{n\times d}$ are two random orthonormal matrices drawn independently  from the Haar measure and $D$ is a diagonal matrix of  size $d$, with its diagonal entries chosen uniformly at random from $[0, 10\delta]$. We also generate $Z\in \bbR^{n\times d}$ via the same generative model as $X$ and let $Y = a X+(1-a) Z$ for $a$ changing values from zero to one. As $a$ varies, we compute $\eps^2 = \|YY^\top-XX^\top\|_2$ and then solve for the procrustes problem $\min_{Q \in \cO}
\|Y- XQ\|_2$ using \algref{procrustes}. Figure \ref{fig:procrustes} plots $\|Y-XQ\|_2$ versus $\epsilon$ in the log-log scale, for different values of $\delta = 1, 2, \dotsc, 5,10$. 

Firstly, we observe that the slope of the best fitted line to each curve is very close to 2, indicating that $\|Y-XQ\|_2$ scales as $\eps^2$. 
\ajR{Secondly, since the singular values of $X$ (there are $d = 10$ of them) are drawn uniformly at random from $[0,10\delta]$}, we have that $\|X^\ddagger\|$ changes as $1/\delta$.
As we observe from the plot, for fixed $\eps$, the term $\|Y-XQ\|_2$ is monotone in $\|X^\ddagger\|\sim \delta^{-1}$. These observations are in good match with our theoretical bound in \thmref{procrustes}.

We next compare the procrustes error $\|Y-XQ\|_2$ with the proposed upper bounds~\eqref{mysharper} and \eqref{procrustes} in \thmref{procrustes}. Recall that the upper bound \eqref{mysharper} reads as
\begin{align*}
\min_{Q \in \cO} \|Y - X Q\|_p \le \begin{cases}
\|X^\ddag\| \eps^2 + \left((1 - \|X^\ddag\|^2 \eps^2)^{-1/2} \|X^\ddag\| \eps^2\right) \wedge (d^{1/4} \eps) \,, \quad &\text{ if } \eps \|X^\ddag\| < 1\,,\\
\|X^\ddag\| \eps^2 + d^{1/4} \eps &\text{ otherwise.}
\end{cases}
\end{align*}
Under the same generative model for configurations $X$ and $Y \in \reals^{n\times d}$, with $\delta = 0.1$,
\figref{bounds-proc}(a) plots the procrustes error along with the above upper bound in the log-log scale. The solid part of the red curve corresponds to the regime where $\eps \|X^\ddag\| <1$ and the dashed part refers to the regime where $\eps\|X^\ddag\| > 1$. Likewise, we plot the upper bound~\eqref{procrustes} in black, which assumes $\eps\|X^\ddag\|<\tfrac{1}{\sqrt{2}}$. The part of this upper bound where this assumption is violated is plotted in dashed form. \figref{bounds-proc}(b) depicts the same curves in the regular (non-logarithmic) scale.
In \figref{bounds-proc}(c), we show the ratio of the upper bounds over the computed procrustes error from the simulation.

 \begin{figure}[]
	\centering
    \includegraphics[scale = 0.4]{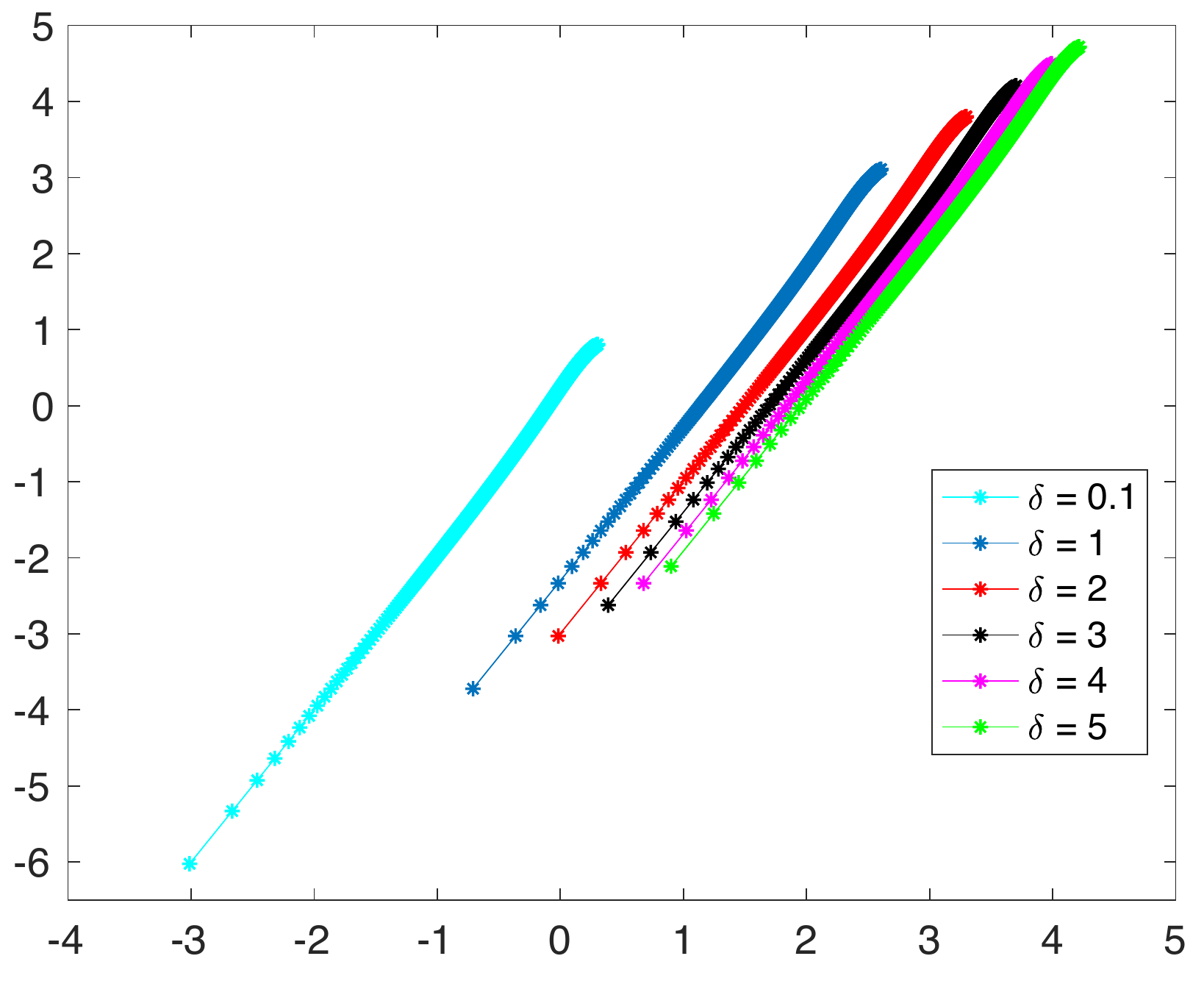}
    \put(-100,-8){{\scriptsize $\log(\eps)$}}
    \put(-200,55){\rotatebox{90}{{\scriptsize $\log(\|Y-XQ\|_2)$}}}
	\caption{\ajR{Procrustes error $\min_{Q\in\cO} \|Y-XQ\|_2$ versus $\epsilon$, in log-log scale and for different values of $\delta$ (i.e., different values of $\|X^\ddagger\|$).}}
	\label{fig:procrustes}
\end{figure}

\begin{figure}[t!]
\captionsetup[subfloat]{captionskip=10pt}
	\centering
    \subfloat[log-log scale]{\includegraphics[scale = 0.4]{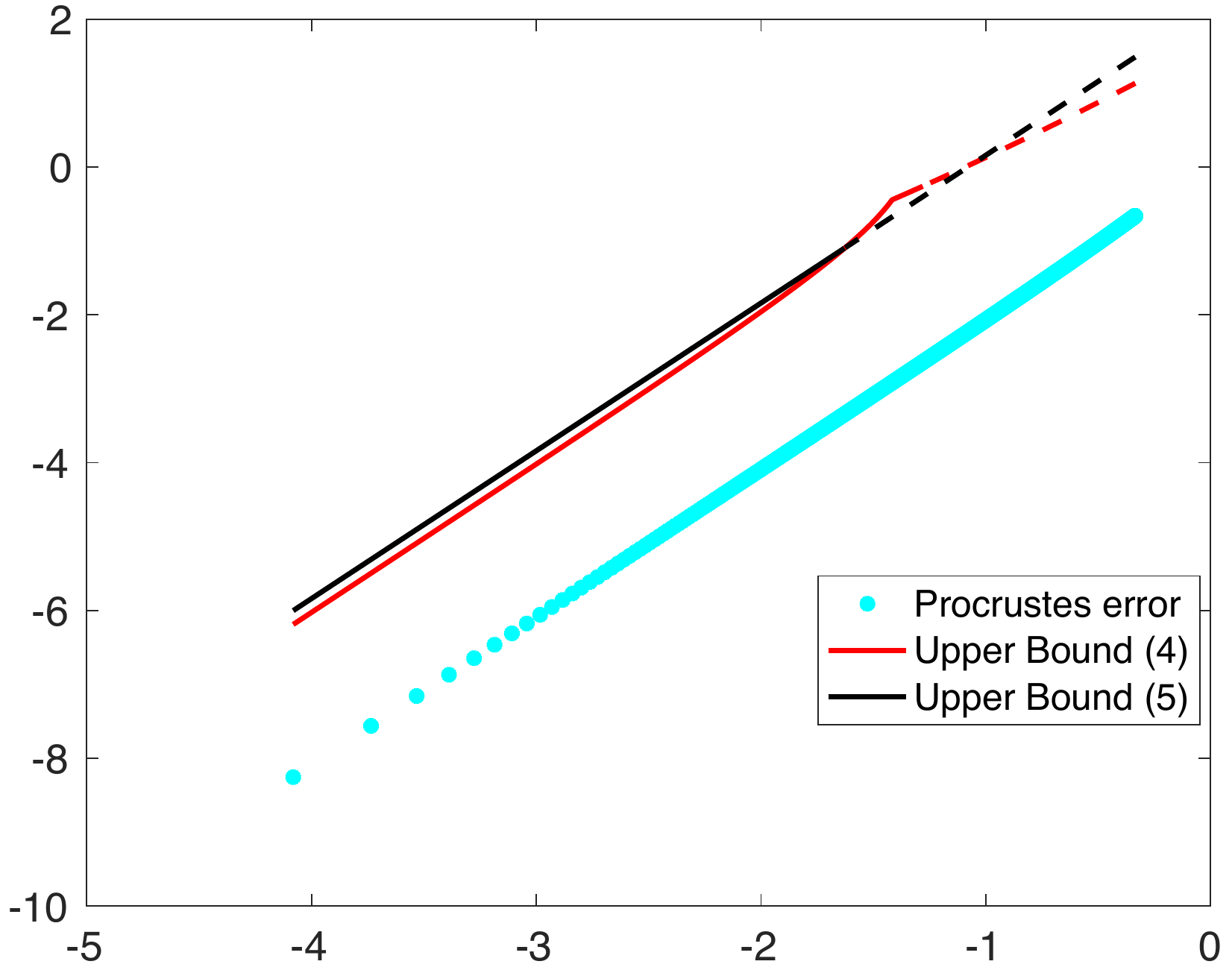} \put(-105,-10){{\scriptsize $\log(\eps)$}}}
    ~\hspace{1.6cm}
    \subfloat[regular (non-logarithmic) scale]{\includegraphics[scale=0.4]{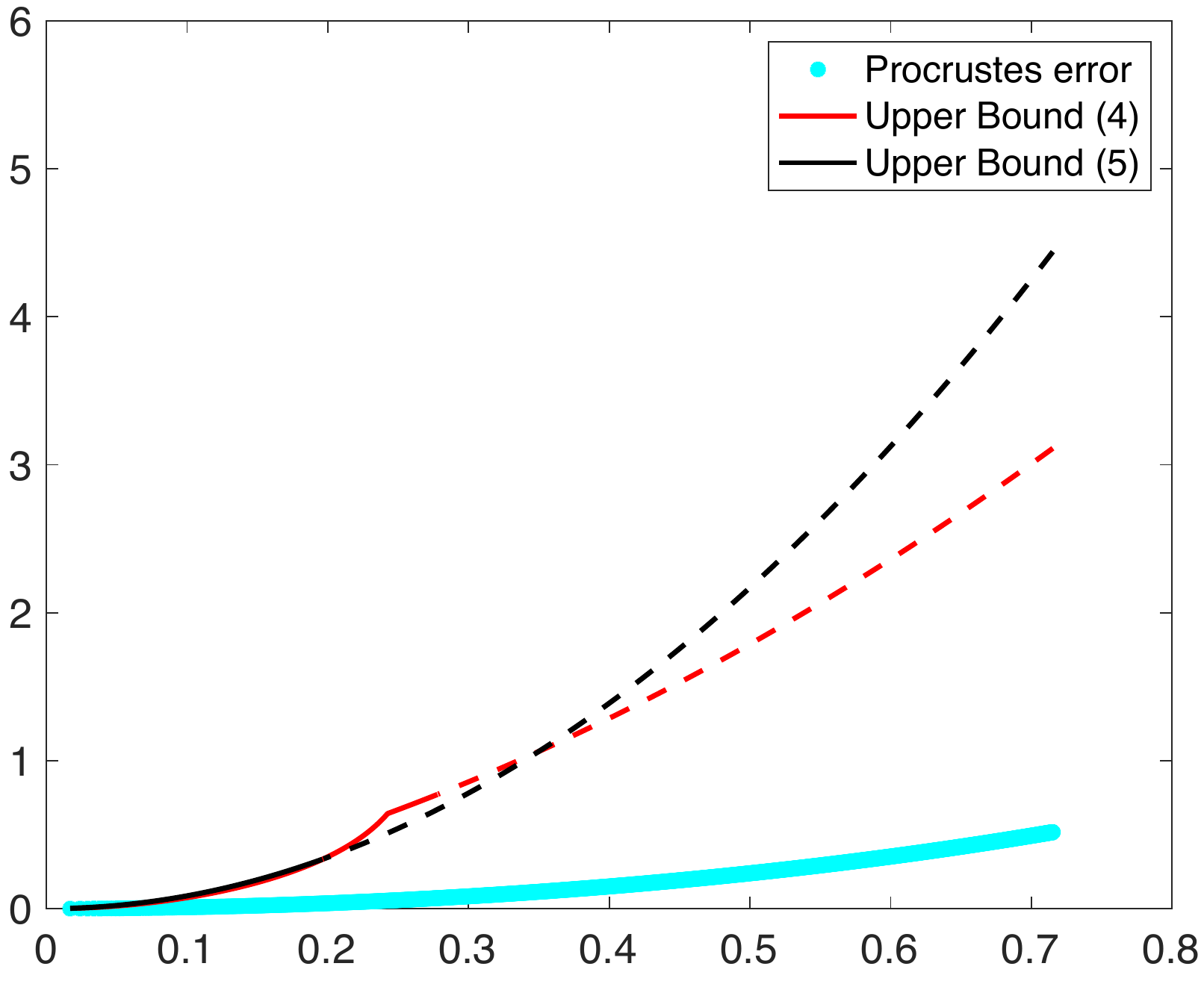}\put(-100,-10){{\scriptsize $\eps$}}}\\
    \subfloat[ratio of upper bounds over the procrustes error]{\includegraphics[scale=0.4]{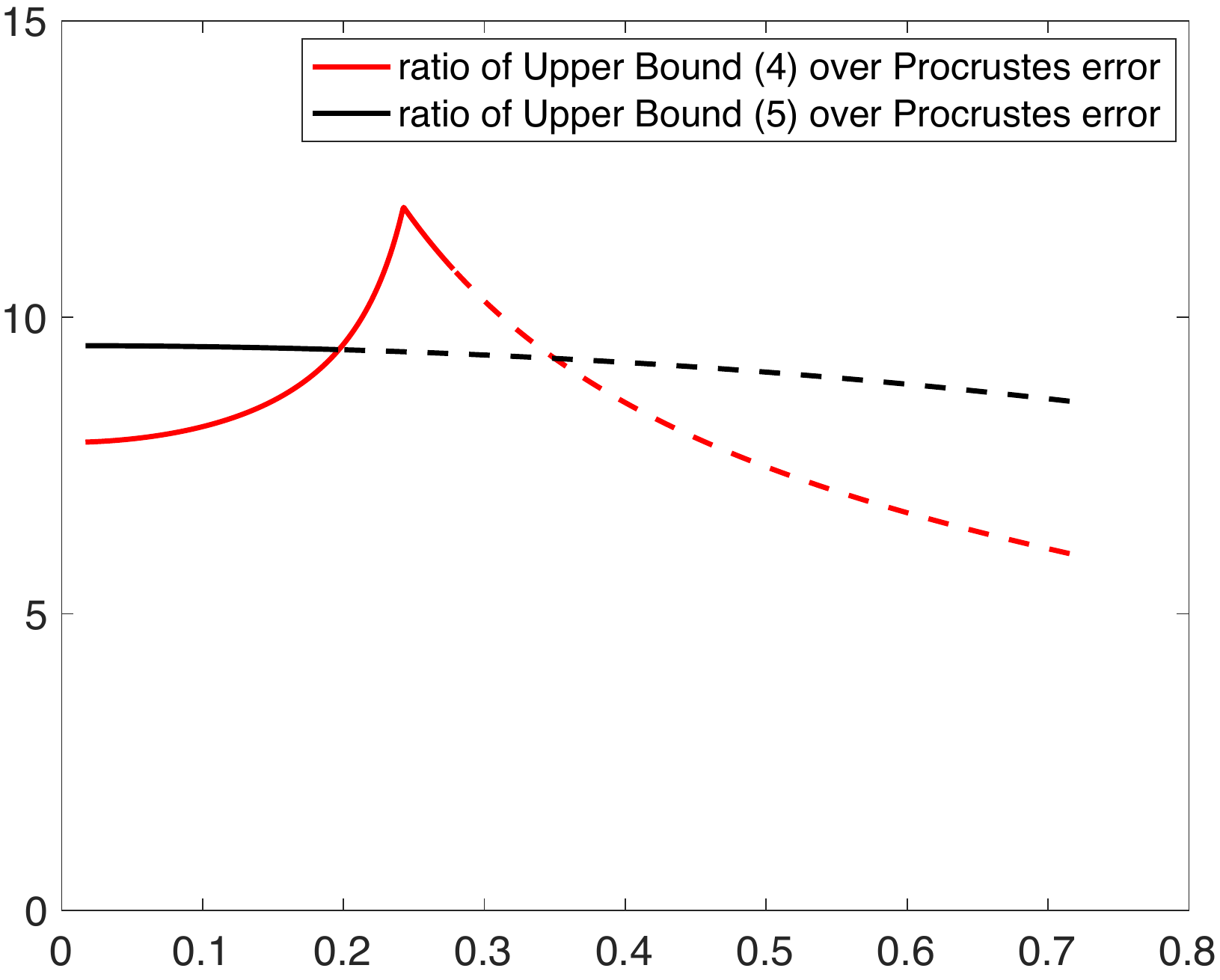}\put(-100,-10){{\scriptsize $\eps$}}}
	\caption{\ajR{Comparison between the procrustes error $\|Y-XQ\|_2$ with the upper bound~\eqref{mysharper} and upper bound~\eqref{procrustes} for the described generative model for configurations $X, Y$ with $\delta = 0.1$; $(a)$ is in log-log scale, $(b)$ is in regular scale; (c) plots the ratio of the upper bounds over the procrustes error. For the red curve (upper bound~\eqref{mysharper}) the solid part corresponds to the regime $\eps\|X^\ddag\| <1$. For the black curve (upper bound~\eqref{procrustes}) the solid part corresponds to the regime where the assumption in deriving this bound, namely $\eps\|X^\ddag\| <\tfrac{1}{\sqrt{2}}$, holds.}}
	\label{fig:bounds-proc}
\end{figure}

\paragraph{Manifold learning algorithms.} 
To evaluate the error rates obtained for manifold learning algorithms in Section \ref{sec:applications}, we carry out two numerical experiments.  

For the first experiment, we consider the `bending map' $\mathcal{B}: [-0.5,0.5]^d \mapsto \bbR^{d+1}$, defined as 
\[
\mathcal{B}(t_1, t_2,\dotsc, t_d) = [R\sin(t_1/R), t_2, \dotsc, t_d, R(1-\cos(t_1/R))]\,.
\]
This map bends the $d$-dimensional hypercube in the $(d+1)$-dimensional space and the parameter $R$ controls the degree of bending (with a large $R$ corresponding to a small amount of bending), and thus controls the reach of the resulting submanifold of $\bbR^{d+1}$. See \figref{bend-a} for an illustration.

We set $R=0.2$ and generate $n$ points $y_1,\dotsc, y_n$ uniformly at random in the $d$-dimensional hypercube. The samples on the manifold are then given by $x_i = \mathcal{B}(y_i)$, for $i=1, \dotsc, n$.
Since the points are well spread out on the manifold, the quantity $a$ given by \eqref{a} is $O(\log(n)/n)^{1/d}$ and following our discussion after the proof of \corref{isomap}, our bound~\eqref{eq:isomap-bound} is optimized at $r\asymp a^{1/2}$. With this choice of $a$, our bound \eqref{eq:isomap-bound} becomes of order $O((\log(n)/n)^{1/d})$.
Following this guideline, we let $r = 2(\log(n)/n)^{1/(2d)}$ and run Isomap (\algref{isomap})  for $d  = 2, 8 , 15$ and $n = 100, 200, \dotsc, 1000$.

Denoting by $z_1,\dotsc, z_n\in \bbR^d$ the output of Isomap in $\bbR^d$, and $Z = [z_1, \dotsc, z_n]^\top\in \bbR^{n\times d}$, $Y = [y_1, \dotsc, y_n]^\top\in \bbR^{n\times d}$, we compute the mismatch between the inferred locations $Z$ and the original ones $Y$ via our metric 
$${{\rm d}}(Y,Z) = \frac{1}{\sqrt{n}} \min_{Q\in \cO} \|Z- YQ\|_2 = \min_{Q \in \cO} \bigg(\frac1n \sum_{i \in [n]} \|z_i - Q y_i\|^2\bigg)^{1/2} \,.$$    
 
 For each $n$, we run the experiment for 50 different realizations of the points in the hypercube and compute the average and the $95\%$ confidence region of the the error ${{\rm d}}(Y,Z)$.
\figref{isomap-bend} reports the results for Isomap in a log-log scale, along with the best linear fits to the data points. 
The slopes of the best fitted lines are $-0.50, -0.14, -0.08$, for $d = 2, 8, 15$, which are close to the corresponding exponent $-\tfrac1d$ implied by our \corref{isomap}, namely,  -0.50,  -0.125, -0.067 (ignoring logarithmic factors). 

Likewise, \figref{MVU-bend} shows the error for Maximum Variance Unfolding (MVU) in the same experiment. As we see, MVU is achieving lower error rates than Isomap. Also the slopes of the best fitted lines are $-0.47, -0.12, -0.04$, for $d = 2, 8, 15$, which are in good agreement with our error rate ($O(\sqrt{d} n^{-1/d})$) in \corref{mvu}. 

In the second experiment, we consider the Swiss Roll manifold, which is a prototypical example in manifold learning. Specifically we consider the mapping $\mathcal{T}: [-\tfrac{9\pi}{2},\tfrac{15\pi}{2}]\times[-40,40]  \mapsto \reals^3$, given by
\begin{align}
\mathcal{T}(t_1,t_2) = [t_1\cos(t_1), t_2, t_1\sin(t_1)]\,.
 \end{align}
The range of this mapping is a Swiss Roll manifold (see \figref{swiss} for an illustration.) For this experiment, we consider non-uniform samples from the manifold as follows. For each $n$, we keep drawing points with first coordinate $\sim \normal(1,\sigma^2)$ and the second coordinate $\sim \normal(0,(10\sigma)^2)$, for a pre-determined value of $\sigma$. If the generated point falls in the rectangle $[-\tfrac{9\pi}{2},\tfrac{15\pi}{2}]\times[-40,40]$, we keep that otherwise reject it. We continue this procedure until we generate $n$ points $y_1, \dotsc, y_n$. The samples on the manifold are given by $x_i = \mathcal{T}(y_i)$. The parameter $\sigma$ controls the dispersion of the samples on the manifold. 

We run Isomap and MVU to infer the underlying positions $y_i$ from the samples $x_i$ on the manifold. For each $\sigma = 0.5, 1, 2$ and $n= 100, 200, \dotsc, 1000$, we run the experiment 50 times and compute the average error ${\rm d}(Y,Z)$ and the $95\%$ confidence region. The results are reported in \figref{swiss-tot} in a log-log scale. As we see the error curves for both algorithms scales as $\sim n^{-1/2}$ for various choice of $\sigma$, which again supports our theoretical error rates stated in Section~\ref{sec:applications}.  
 
 \vspace{0.5cm}

 \begin{figure}[t!]
 \captionsetup[subfloat]{captionskip=20pt}
	\centering
    \subfloat[Bended square ($d=2$, $R = 0.2$)\label{fig:bend-a}]{\includegraphics[scale = 0.4]{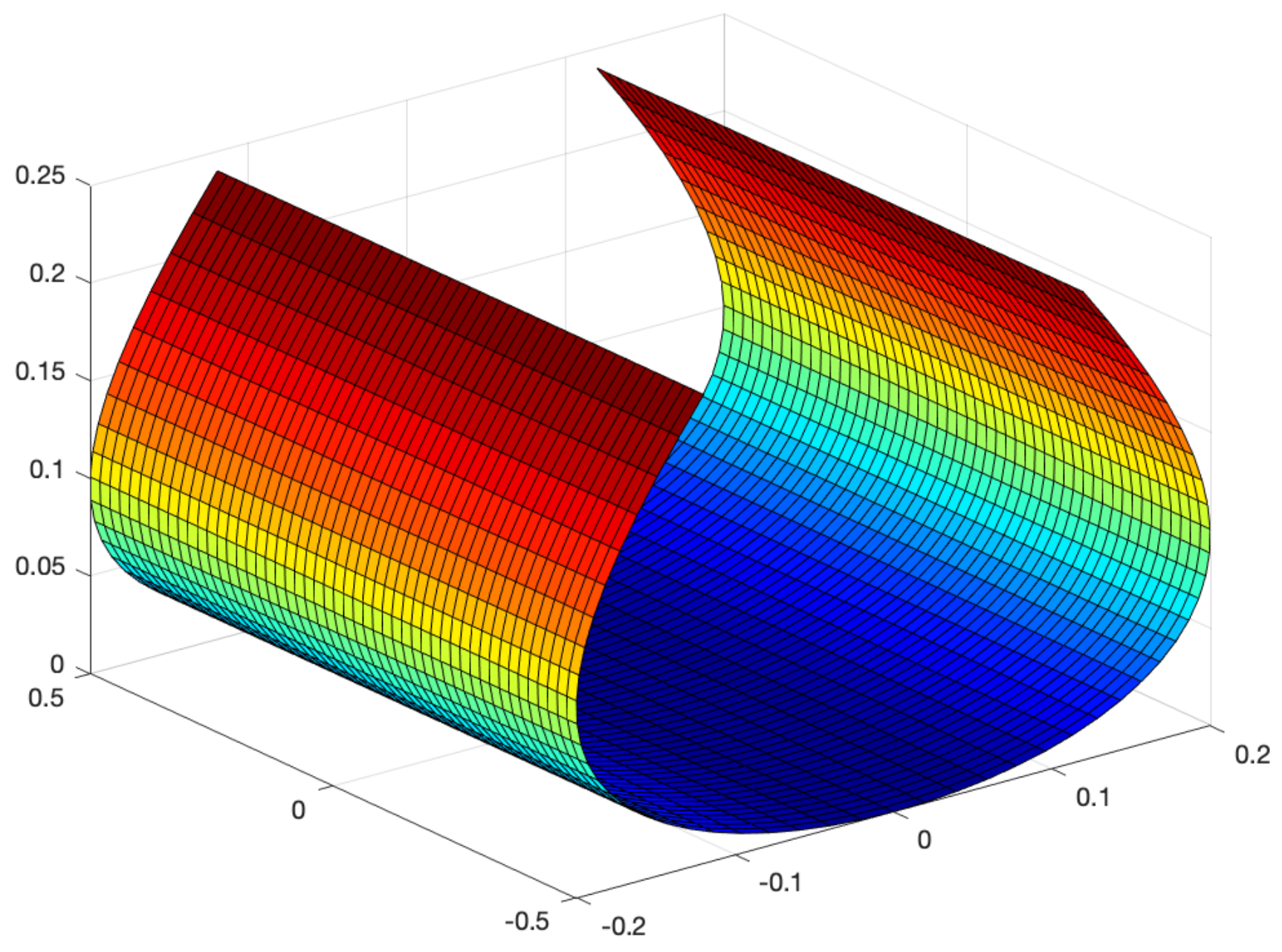} }
    ~\hspace{1.6cm}
    \subfloat[Isomap\label{fig:isomap-bend}]{\includegraphics[scale=0.4]{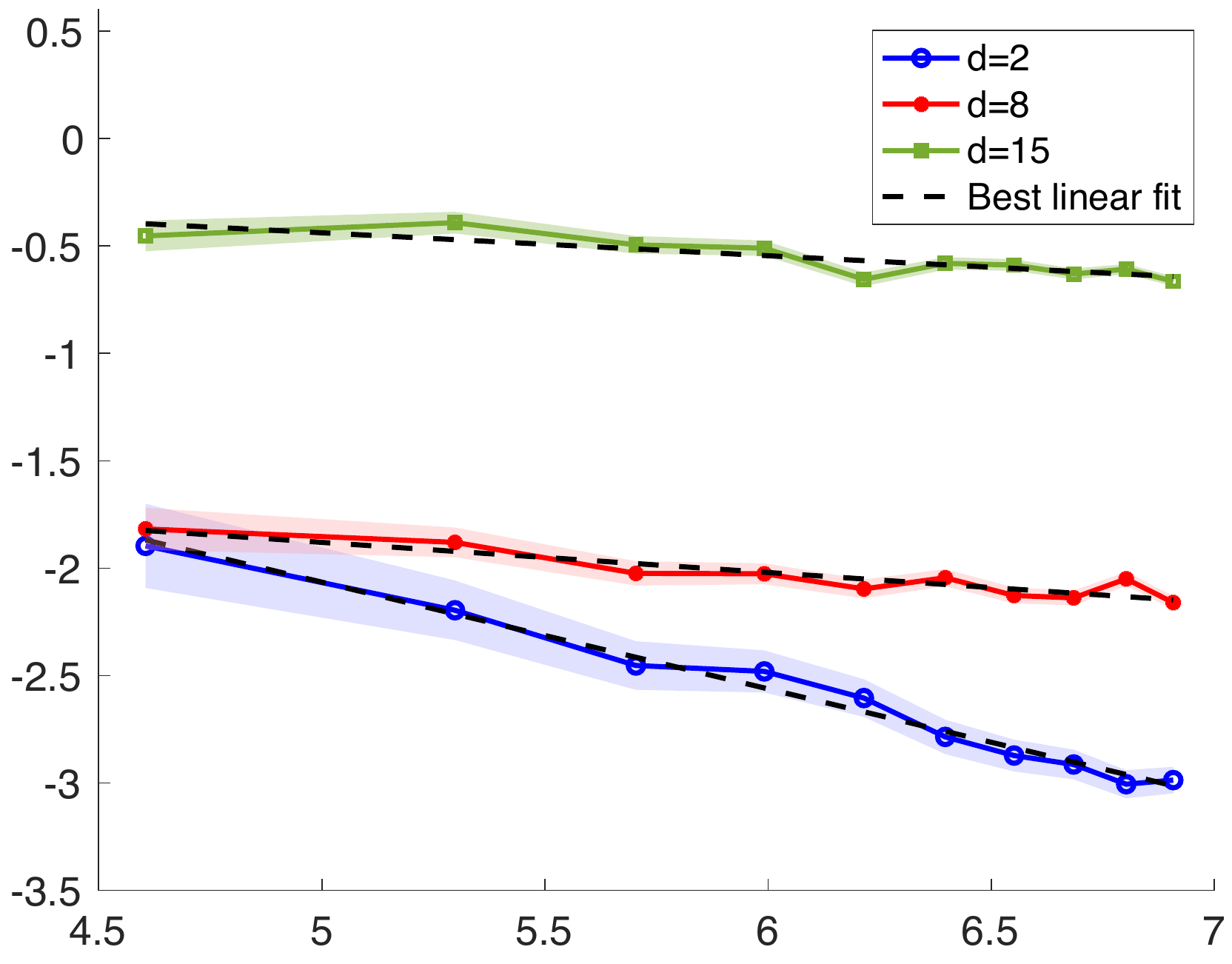}}
    \put(-205,60){\rotatebox{90}{{\scriptsize $\log(d(Y,Z))$}}}
    \put(-103,-10){{\scriptsize $\log(n)$}}
    \\
    \subfloat[Maximum Variance Unfolding\label{fig:MVU-bend}]{\includegraphics[scale=0.4]{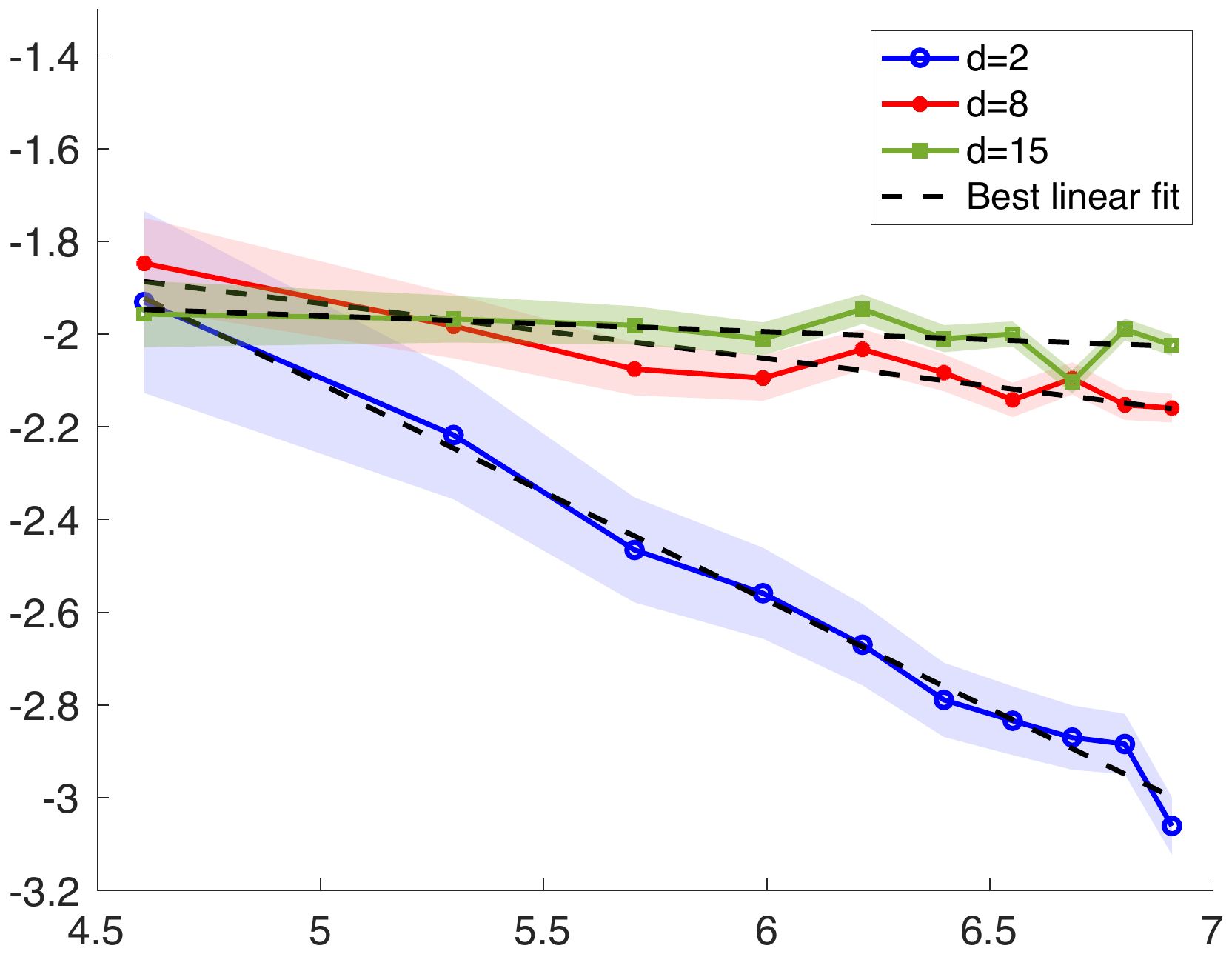}}
    \put(-200,55){\rotatebox{90}{{\scriptsize $\log(d(Y,Z))$}}}
    \put(-103,-10){{\scriptsize $\log(n)$}}
	\caption{\ajR{Performance of Isomap and MVU on the data points sampled from the bended hypercube of dimension $d$. Different curves correspond to different values of $d$. Each curve is plotted along with the corresponding best fitted line and the 95\% confidence region.}}
	\label{fig:bend}
\end{figure}


 \begin{figure}[t!]
 \captionsetup[subfloat]{captionskip=20pt}
	\centering
    \subfloat[Swiss roll manifold\label{fig:swiss}]{\includegraphics[scale = 0.4]{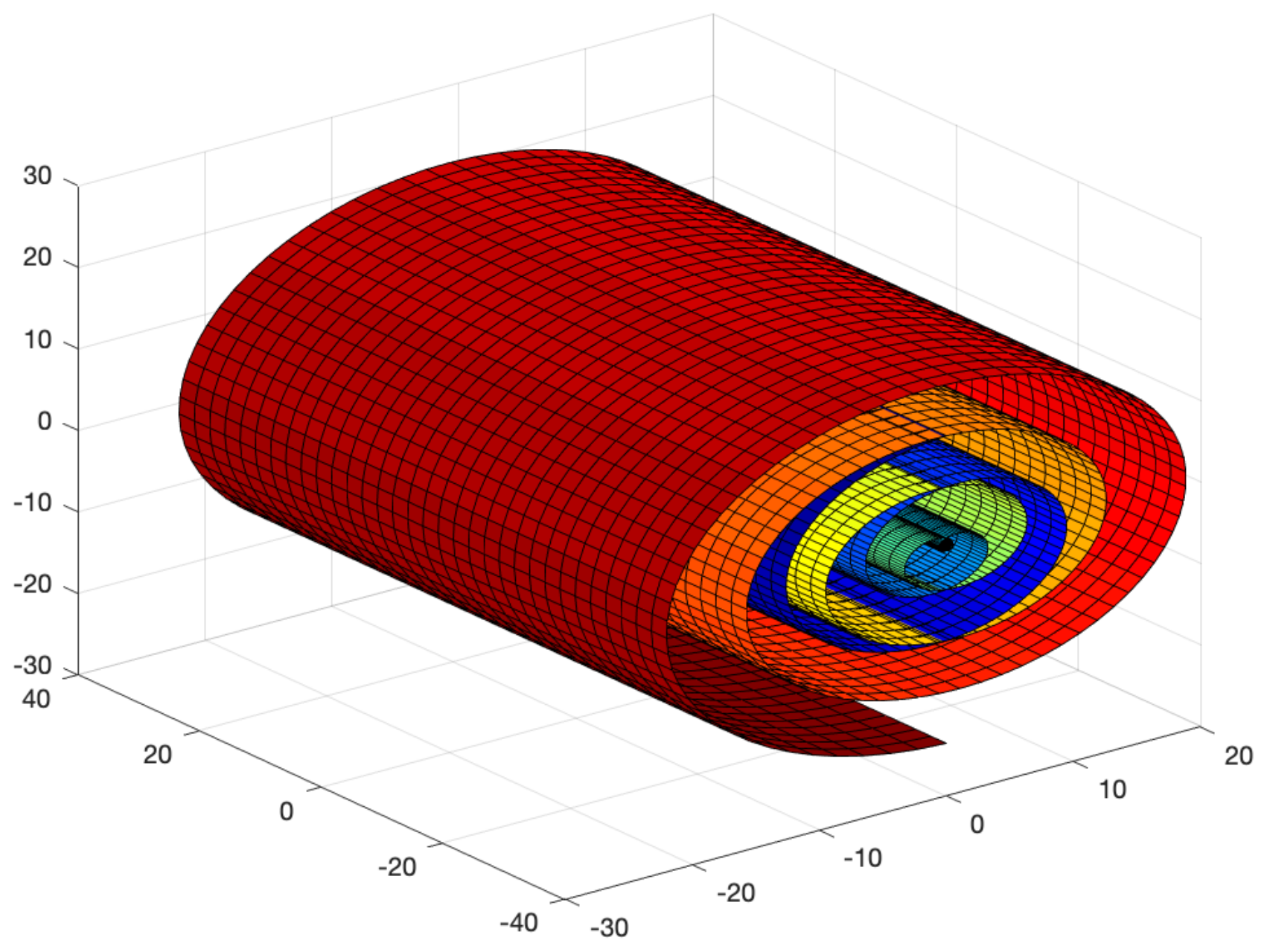}}
    ~\hspace{1.6cm}
    \subfloat[Isomap]{\includegraphics[scale=0.4]{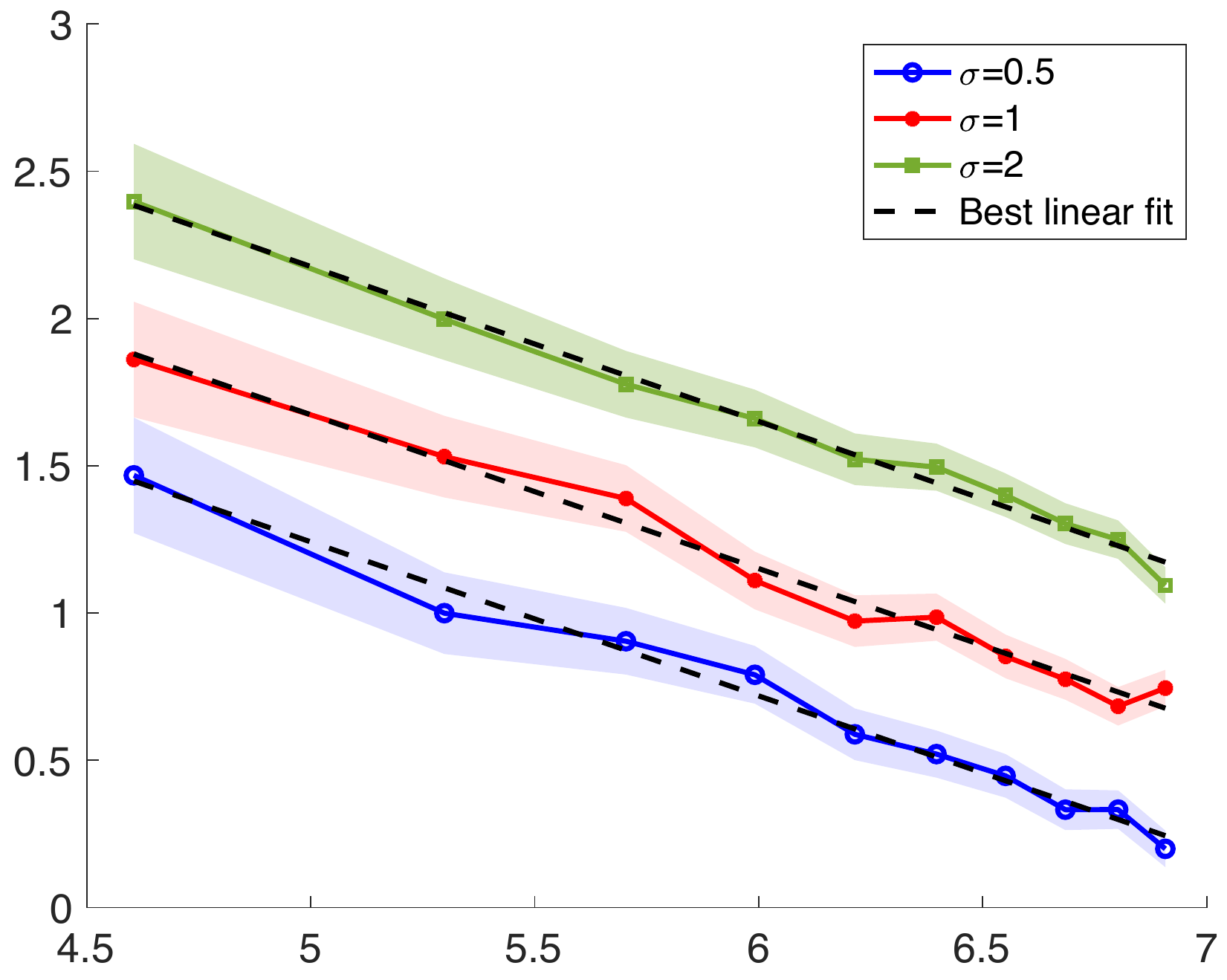}} 
    \put(-205,60){\rotatebox{90}{{\scriptsize $\log(d(Y,Z))$}}}
    \put(-103,-10){{\scriptsize $\log(n)$}}
    \\
    \subfloat[Maximum Variance Unfolding]{\includegraphics[scale=0.4]{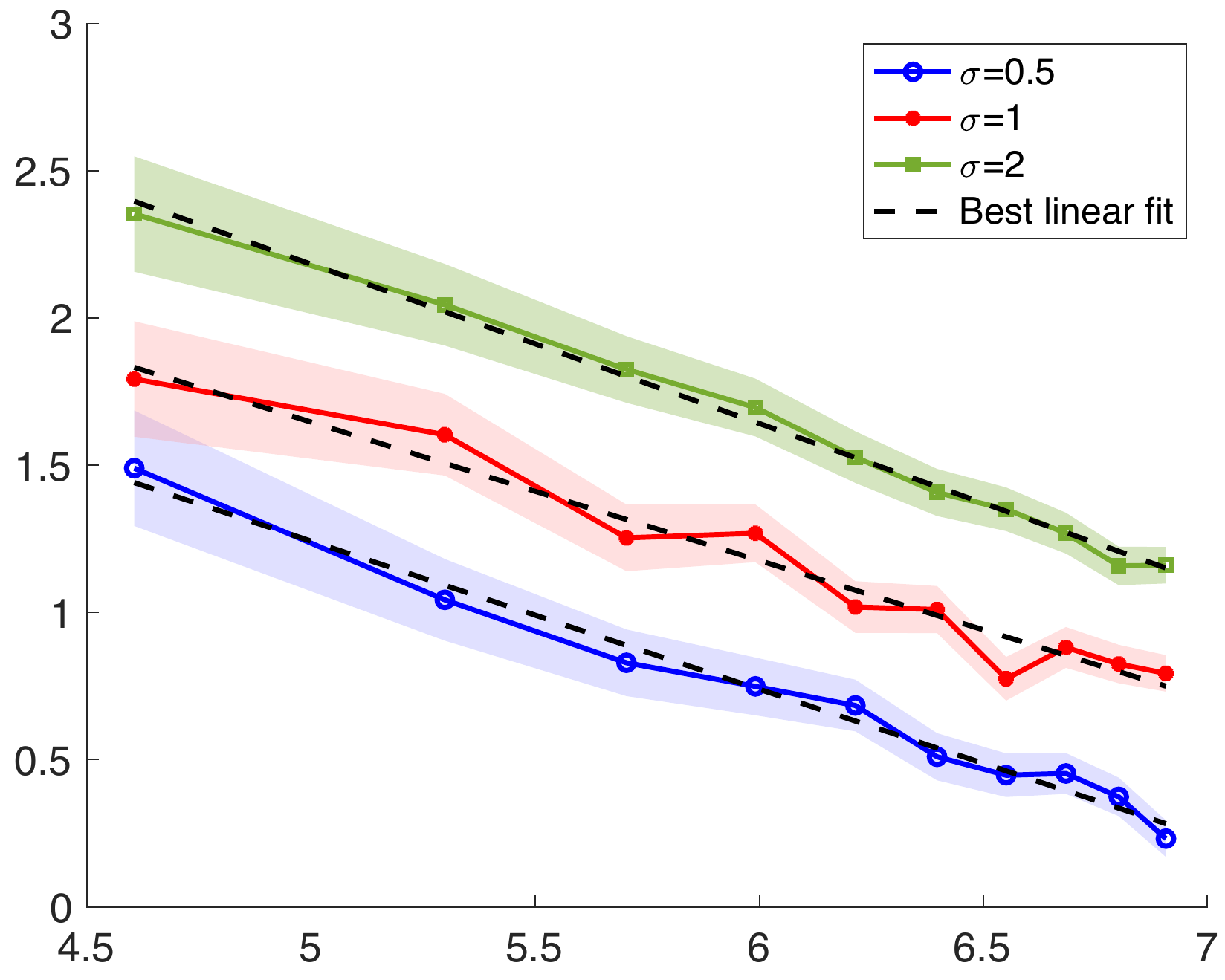}}
    \put(-200,55){\rotatebox{90}{{\scriptsize $\log(d(Y,Z))$}}}
    \put(-103,-10){{\scriptsize $\log(n)$}}
	\caption{\ajR{Performance of Isomap and MVU on the data points sampled non-uniformly from the Swiss Roll manifold. Different curves correspond to different values of $\sigma$, which controls the dispersion of the sampled points. Each curve is plotted along with the corresponding best fitted line and the 95\% confidence region.}}
	\label{fig:swiss-tot}
\end{figure}
}
\section{Discussion}
\label{sec:discussion}
\subsection{Optimality considerations}
The performance bounds that we derive for Isomap and Maximum Variance Unfolding are the same up to a universal multiplicative constant.  This may not be surprising as they are known to be closely related, since the work of \citet{paprotny2012connection}.  Based on our analysis of classical scaling, we believe that the bound for Isomap is sharp up to a multiplicative constant.  But one may wonder if Maximum Variance Unfolding, or a totally different method, can do strictly better.

This optimality problem can be formalized as follows:  

Consider the class of isometries $\varphi: \cD \to \M \subset \bbR^D$, one-to-one, such that its domain $\cD$ is a convex subset of $\bbR^d$ with max-radius at most $\rho_0$ and half-width at least $\omega_0 > 0$, and its range $\M$ is a submanifold with reach at least $\tau_0 > 0$.  To each such isometry $\varphi$, we associate the uniform distribution on its range $\M$, denoted by ${\sf P}_\varphi$.  We then assume that we are provided with iid samples of size $n$ from ${\sf P}_\varphi$, for some unknown isometry $\varphi$ in that class.  If the sample is denoted by $x_1, \dots, x_n \in \M$, with $x_i = \varphi(y_i)$ for some $y_i \in \cD$, the goal is to recover $y_1, \dots, y_n$ up to a rigid transformation, and the performance is measured in average squared error. Then, what is the optimal achievable performance?

Despite some closely related work on manifold estimation, in particular work of \citet{genovese2012manifold} and of \citet{kim2015tight}, we believe the problem remains open.  Indeed, while in the setting in dimension $d=1$ the two problems are particularly close, in dimension $d \ge 2$ the situation here appears more delicate here, as it relies on a good understanding of the interpolation of points by isometries.

\subsection{Choosing landmarks}
\label{sec:app_landmarks}
In this subsection we discuss the choice of landmarks.  
We consider the two methods originally proposed by \citet{de2004sparse}:
\bitem
\item {\sf Random.}  
The landmarks are chosen uniformly at random from the data points.
\item {\sf MaxMin.}  
After choosing the first landmark uniformly at random from the data points, each new landmark is iteratively chosen from the data points to maximize the minimum distance to the existing landmarks. 
\eitem
(For both methods, \citet{de2004sparse} recommend using different initializations.)

The first method is obviously less computationally intensive compared to the second method, but the hope in the more careful (and also more costly) selection of landmarks in the second method is that it would require fewer landmarks to be selected.  In any case, \citet{de2004sparse} observe that the random selection is typically good enough in practice, so we content ourselves with analyzing this method.  

In view of our findings (\corref{trilateration}, \ref{cor:isomap}, \ref{cor:landmark isomap}, and \ref{cor:mvu}), a good choice of landmarks is one that has large (embedded) half-width, ideally comparable to, or even larger than that of the entire dataset.
In that light, the problem of selecting good landmarks is closely related, if not identical, to problem of selecting rows of a tall matrix in a way that leads to a submatrix with good condition number.  In particular, several papers have established bounds for various ways of selecting the rows, some of them listed in \cite[Tab 2]{holodnak2015randomized}.  Here the situation is a little different in that the dissimilarity matrix is not directly available, but rather, rows (corresponding to landmarks) are revealed as they are selected.

The {\sf Random} method, nonetheless, has been studied in the literature.  Rather than fetch existing results, we provide a proof for the sake of completeness.  As everyone else, we use random matrix concentration \cite{tropp2012user}.  We establish a bound for a slightly different variant where the landmarks are selected with replacement, as it simplifies the analysis.  Related work is summarized in \cite[Tab 3]{holodnak2015randomized}, although for the special case where the data matrix (denoted $Y$ earlier) has orthonormal columns (an example of paper working in this setting is \cite{ipsen2014effect}).

\begin{prp}\label{prp:random landmark}
Suppose we select $\ell$ landmarks among $n$ points in dimension $d$, with half-width $\omega$ and max-radius $\rho_\infty$, according to the {\sf Random} method, but with replacement.  Then with probability at least $1 - 2 (d+1) \exp[-\ell \omega^2/9\rho_\infty^2]$, the half-width of the selected landmarks is at least $\omega/2$.
\end{prp}

The proof of \prpref{random landmark} is given in \secref{proof random landmark}.
Thus, if $\ell \ge 9 (\rho_\infty/\omega)^2 \log(2(d+1)/\delta)$, then with probability at least $1-\delta$ the landmark set has half-width at least $\omega/2$.  Consequently, if the dataset is relatively well-conditioned in that its aspect ratio, $\rho_\infty/\omega$, is relatively small, then {\sf Random} (with replacement) only requires the selection of a few landmarks in order to output a well-conditioned subset (with high probability). 

%


\section{Proofs}
\label{sec:proofs}

\subsection{Preliminaries}
\label{sec:preliminaries}
We start by stating a number of lemmas pertaining to linear algebra and end the section with a result for a form of  procrustes analysis, a well-known method for matching two sets of points in a Euclidean space.  

\paragraph{Schatten norms}
For a matrix\footnote{ All the matrices and vectors we consider are real, unless otherwise specified.} $A$, we let $\nu_1(A) \ge \nu_2(A) \ge \cdots$ denote its singular values.
Let $\|\cdot\|_p$ denote the following Schatten quasi-norm,
\beq
\|A\|_p \equiv \big(\nu_1(A)^p + \cdots + \nu_d(A)^p)^{1/p},
\eeq
which is a true norm when $p \in [1, \infty]$.  When $p = 2$ it corresponds to the Frobenius norm (which will also be denoted by $\|\cdot\|_2$) and when $p = \infty$ it corresponds to the usual operator norm (which will also be denoted by $\|\cdot\|$).
We mention that each Schatten quasi-norm is unitary invariant, and satisfies
\beq\label{norm}
\|AB\|_p \le \|A\|_\infty \|B\|_p,
\eeq
for any matrices of compatible sizes, and it is sub-multiplicative if it is a norm ($p \ge 1$).
In addition, $\|A\|_p = \|A^\top\|_p$ and $\|A\|_p = \|A^\top A\|_{p/2}^{1/2} = \|A A^\top\|_{p/2}^{1/2}$, due to the fact that 
\beq
\|A\|_p^p
= \sum_j \nu_j(A)^p
= \sum_j \nu_j(A^\top A)^{p/2}
= \|A^\top A\|_{p/2}^{p/2}\,.
\eeq
and if $A$ and $B$ are positive semidefinite satisfying $A \preceq B$, where $\preceq$ denotes the Loewner order, then $\|A\|_p \le \|B\|_p$.
\ajr{
To see this, note that by definition $A \preceq B$ means $0\preceq B-A$, and so $0\le v^\top(B-A)v$ for any vector $v$. Therefore, by using the variational principle of eigenvalues (min-max Courant-Fischer theorem) we have
\begin{align*}
\nu_j(A) &= \underset{V, {{\rm dim}}(V) = n-j+1} {\min} \;\; \underset{v\in V, \|v\|=1}{\max}\;\; v^\top A v\\
&\le \underset{V, {{\rm dim}}(V) = n-j+1}{\min} \;\; \underset{v\in V, \|v\|=1}{\max}\;\; v^\top B v = \nu_j(B)\,,
\end{align*}
for all $j$. As a result, $\|A\|_p \le \|B\|_p$.
We refer the reader to \cite{bhatia2013matrix} for more details on the Schatten norms and the Loewner ordering on positive semidefinite matrices.
}

Unless otherwise specified, $p$ will be fixed in $[1, \infty]$.
Note that, for any fixed matrix $A$, $\|A\|_p \le \|A\|_q$ whenever $q \le p$, and 
\beq\label{continuity}
\|A\|_p \to \|A\|_\infty, \quad p \to \infty.
\eeq

\paragraph{Moore-Penrose pseudo-inverse}
The Moore-Penrose pseudo-inverse of a matrix is defined as follows \cite[Thm III.1]{MR1061154}.  
Let $A$ be a $m$-by-$k$ matrix, where $m \ge k$, with singular value decomposition $A = U D V^\top$, where $U$ is $m$-by-$k$ orthogonal, $V$ is $k$-by-$k$ orthogonal, and $D$ is $k$-by-$k$ diagonal with diagonal entries $\nu_1 \ge \cdots \ge \nu_l > 0 = \cdots = 0$, so that the $\nu_j$'s are the nonzero singular values of $A$ and $A$ has rank $l$.  
The pseudo-inverse of $A$ is defined as $A^\ddag = V D^\ddag U^\top$, where $D^\ddag = \diag(\nu_1^{-1}, \dots, \nu_l^{-1}, 0, \dots, 0)$.  
If the matrix $A$ is tall and full rank, then $A^\ddag = (A^\top A)^{-1} A^\top$.  In particular, if a matrix is square and non-singular, its pseudo-inverse coincides with its inverse.

\begin{lem}\label{lem:pseudo}
Suppose that $A$ is a tall matrix with full rank.  Then $A^\ddag$ is non-singular, and for any other matrix $B$ of compatible size,
\beq\label{pseudo}
\|B\|_p \le \|A^\ddag\|_\infty \|AB\|_p.
\eeq
\end{lem}

\begin{proof}
This simply comes from the fact that $A^\ddag A = I$ (since $A$ is tall and full rank), so that
\beq
\|B\|_p = \|A^\ddag A B\|_p \le \|A^\ddag\|_\infty \|AB\|_p,
\eeq
by \eqref{norm}.
\end{proof}

\begin{lem} \label{lem:wedin}
Let $A$ and $B$ be matrices of same size.  Then, for $p \in \{2, \infty\}$,
\beq\label{wedin}
\|B^\ddag - A^\ddag\|_p \le \frac{\sqrt{2} \|A^\ddag\|^2 \|B - A\|_p}{(1 - \|A^\ddag\| \|B - A\|)_+^2}.
\eeq
\end{lem}

\begin{proof}
A result of Wedin \cite[Thm III.3.8]{MR1061154} gives \footnote{\ajr{For $p = \infty$, the factor $\sqrt{2}$ in \eqref{eq:Lema3-1} can be removed, giving a tighter bound in this case.}}
\beq\label{eq:Lema3-1}
\|B^\ddag - A^\ddag\|_p \le \sqrt{2}\, (\|B^\ddag\| \vee \|A^\ddag\|)^2\, \|B - A\|_p, \quad p \in \{2, \infty\}.
\eeq
Assuming $B$ has exactly $k$ nonzero singular values, using Mirsky's inequality \cite[Thm IV.4.11]{MR1061154}, namely 
\beq\label{mirsky}
\max_j |\nu_j(B) - \nu_j(A)| \le \|B - A\|,
\eeq
we have
\beq\label{eq:Lema3-2}
\|B^\ddag\|^{-1} = \nu_k(B) \ge (\nu_k(A) - \|B - A\|)_+ \ge (\|A^\ddag\|^{-1} - \|B - A\|)_+.
\eeq
By combining Equations~\eqref{eq:Lema3-1} and~\eqref{eq:Lema3-2}, we get
\beq
\|B^\ddag - A^\ddag\|_p \le \sqrt{2} \left(\|A^\ddag\| \vee \frac{1}{(\|A^\ddag\|^{-1} - \|B - A\|)_+}\right)^2 \|B - A\|_p\,,
\eeq
from which the result follows.
\end{proof}

\paragraph{Some elementary matrix inequalities}
The following lemmas are elementary inequalities involving Schatten norms.

\begin{lem} \label{lem:perp-matrices}
For any two matrices $A$ and $B$ of same size such that $A^\top B = 0$ or $A B^\top = 0$,
\beq\label{perp-matrices}
\|A + B\|_p \ge \|A\|_p \vee \|B\|_p.
\eeq
\end{lem}

\begin{proof}
Assume without loss of generality that $A^\top B = 0$.  
In that case, $(A+B)^\top (A+B) = A^\top A + B^\top B$, which is not smaller than $A^\top A$ or $B^\top B$ in the Loewner  order.  
Therefore,
\begin{align}
\|A\|_p
= \|A^\top A\|_{p/2}^{1/2}
&\le \|A^\top A + B^\top B\|_{p/2}^{1/2} \\
&= \|(A+B)^\top (A+B)\|_{p/2}^{1/2}
= \|A+B\|_p,
\end{align}
applying several of the properties listed above for Schatten (quasi)norms.
\end{proof}

\ajR{

\begin{lem}\label{lem:A-II}
For any matrix $A$ and any positive semidefinite matrix $B$, we have
\beq\label{A-II}
\|A\|_p \le \|A(B+I)\|_p,
\eeq
where $I$ denotes the identity matrix, with the same dimension as $B$ .
\end{lem}

\begin{proof}
We write
\[
A(B+I) (B+I)^\top A^\top = A(B^2+2B+I)A^\top = AA^\top + A(B^2+2B) A^\top,
\]
with $A(B^2+2B) A^\top \succeq 0$. Therefore, for all $k$,
\[
\nu_k(A(B+I) (B+I)^\top A^\top) \ge \nu_k(AA^\top),
\]
which then implies that $\nu_k(A(B + I)) \ge \nu_k(A)$ for all $k$, which finally yields the result from the mere  definition of the $p$-Schatten norm.
\end{proof}
}

%

\ajR{
\subsection{Proof of \thmref{procrustes}}
\label{sec:proof_procrustes}
Suppose $X, Y\in \mathbb{R}^{n\times d}$  and 
  let $P\in \mathbb{R}^{n\times n}$ be the orthogonal projection onto the column space of $X$, which can be expressed as $P = X X^\ddag$.  
Define $Y_1 = PY$ and $Y_2 = (I-P)Y$, and note that $Y = Y_1 + Y_2$ with $Y_2^\top Y_1 = 0$, and also $Y_2^\top X = 0$.  

Define $M = X^\ddag Y\in \mathbb{R}^{d\times d}$, and apply a singular value decomposition to obtain $M = U D V^\top$, where $U$ and $V$ are orthogonal matrices of size $d$, and $D$ is diagonal with nonnegative entries. Indeed columns of $U$ span the row space of $X$ and columns of $V$ span the row space of $Y$. {Then define $Q = U V^\top$, which is orthogonal.}
We show that the bound \eqref{procrustes} holds for this orthogonal matrix.

We start with the triangle inequality,
\begin{align}\label{eq:dec}
\|Y - X Q\|_p 
= \|Y_1 - XQ + Y_2\|_p 
\le \|Y_1 - XQ\|_p + \|Y_2\|_p.
\end{align}
Noting that $Y_1 = X X^\ddag Y = X M$, we have
\begin{align}
\|Y_1 - XQ\|_p 
= \|X M - X Q\|_p 
&= \|X U D V^\top - X U V^\top\|_p \nonumber\\
&= \|X U (D-I) V^\top\|_p 
\le \|X U (D-I)\|_p.\label{eq:U1}
\end{align}
Now by Lemma~\ref{lem:A-II}, we have
\begin{align}
\|XU(D-I)\|_p \le \|XU(D-I) (D+I)\|_p = \|XU(D^2-I)\|_p\,.\label{eq:U2}
\end{align}
Now by unitary invariance, we have
\begin{align}
\|XU(D^2-I)\|_p =\|XU(D^2-I)U^\top\|_p = \|XUD^2U^\top - XUU^\top\|_p = \|XUD^2U^\top - X\|_p\,,\label{eq:U3}
\end{align}
where in the last step we used the fact that columns of $U$ span the row space of $X$ and hence $UU^\top X^\top = X^\top$.
Combining~\eqref{eq:U1}, \eqref{eq:U2} and \eqref{eq:U3}, we obtain
\begin{align}
\|Y_1 - XQ\|_p  &\le  \|XUD^2U^\top - X\|_p \\
& =  \|(XMM^\top - X)(X^\ddagger X)^\top\|_p\\
&\le \|X^\ddagger\| \| XMM^\top X^\top - XX^\top\|_p\\
&=\|X^\ddag\| \|Y_1 Y_1^\top - X X^\top\|_p,
\end{align}
where the first equality holds since $X^\ddagger X = I$, given that $X$ has full column rank.

Coming from the other end, so to speak, we have
\begin{align}
\eps^2 = \|Y Y^\top - X X^\top\|_p
&= \|Y_1 Y_1^\top - X X^\top + Y_1 Y_2^\top + Y_2 Y_1^\top + Y_2 Y_2^\top\|_p \\
&\ge \|Y_1 Y_1^\top - X X^\top + Y_1 Y_2^\top\| \vee \|Y_2 Y_1^\top + Y_2 Y_2^\top\|_p \\
&\ge \|Y_1 Y_1^\top - X X^\top\|_p \vee \|Y_1 Y_2^\top\|_p \vee \|Y_2 Y_1^\top\|_p \vee \|Y_2 Y_2^\top\|_p, \label{delta1}
\end{align}
using \lemref{perp-matrices} thrice, once based on the fact that
\beqn
(Y_1 Y_1^\top - X X^\top + Y_1 Y_2^\top)^\top (Y_2 Y_1^\top + Y_2 Y_2^\top)
= \underbrace{(Y_1 Y_1^\top - X X^\top + Y_2 Y_1^\top)Y_2}_{= 0}(Y_1^\top + Y_2^\top)
= 0,
\eeq
and then based on the fact that
\beqn
(Y_1 Y_1^\top - X X^\top) (Y_1 Y_2^\top)^\top = \underbrace{(Y_1 Y_1^\top - X X^\top) Y_2}_{=0} Y_1^\top = 0,
\eeq
and
\beqn
(Y_2 Y_1^\top) (Y_2 Y_2^\top)^\top = Y_2 \underbrace{Y_1^\top Y_2}_{=0} Y_2^\top.
\eeq

From \eqref{delta1}, we extract the bound $\|Y_1 Y_1^\top - X X^\top\|_p \le \eps^2$, from which we get (based on the derivations above)
\beq\label{1st}
\|Y_1 - X Q\|_p 
\le \|X^\ddag\| \eps^2.
\eeq

 Recalling the inequality~\eqref{eq:dec}, we proceed to bound $\|Y_2\|_p$.  From \eqref{delta1}, we extract the bound $\|Y_2 Y_2^\top\|_p \le \eps^2$, and combine it with
\beqn
\|Y_2 Y_2^\top\|_p
= \|Y_2\|_{2p}^2
\ge d^{-1/p} \|Y_2\|_p^2\,,
\eeq
where $d$ is the number of columns and the inequality is Cauchy-Schwarz's, to get 
\beqn\label{Y2-general}
\|Y_2\|_p \le d^{1/2p} \eps\,.
\eeq
 
 We next derive another upper bound for $\|Y_2\|_p$, for the case that $\|X^\ddag\|\eps< 1$.
 Denote by $\lambda_1\ge\dotsc\ge\lambda_d$ be the singular values of $X$ and by $\nu_1\ge \dotsc\ge \nu_d$ the singular values of $Y_1$. Given that $X$ has full column rank we have $\lambda_d >0$ and so $\|X^\ddagger\| = 1/\lambda_d$. Further, by an application of Mirsky's inequality \cite[Thm IV.4.11]{MR1061154}, we have
\[
\max_i|\nu_i^2 - \lambda_i^2|\le \|Y_1 Y_1^\top - XX^\top\| \le  \|Y_1 Y_1^\top - XX^\top\|_p \le \eps^2,
\]
using Equation~\eqref{delta1}. Therefore $\nu_d^2> \lambda_d^2 - \eps^2 >0$ 
 by our assumption that $\|X^\ddag\| \eps^2 < 1$, which 
implies that $Y_1$ has full column rank. Now, by an application of Lemma~\ref{lem:pseudo}, we obtain
\beq\label{eq:Y2}
\|Y_2\|_p = \|Y_2^\top\|_p \le \|Y_1^\ddagger\| \|Y_1 Y_2^\top\|_p \le \eps^2\|Y_1^\ddagger\|\,,
\eeq
 where we used \eqref{delta1} in the last step. Also,
 \beq\label{eq:Y1}
 \|Y_1^\ddagger\| =  \frac{1}{\nu_d} \le \frac{1}{(\lambda_d^2 - \eps^2)^{1/2}} = \frac{\lambda_d^{-1}}{(1 - \eps^2 \lambda_d^{-2})^{1/2}}
 = \|X^\ddagger\| (1 - \eps^2\|X^\ddagger\|^2)^{-1/2}\,. 
 \eeq
 Combining~\eqref{eq:Y1} and \eqref{eq:Y2} we obtain
 \beq\label{eq:Y2-final}
 \|Y_2\|_p\le \eps^2 \|X^\ddagger\| (1 - \eps^2\|X^\ddagger\|^2)^{-1/2}\,, \quad \quad \text{ if }\;\; \|X^\ddag\|\eps<1\,. 
 \eeq

Combining the the bounds~\eqref{Y2-general} with \eqref{eq:Y2-final} and \eqref{1st} in the inequality~\eqref{eq:dec}, we get \eqref{mysharper}.
The bound~\eqref{procrustes} follows readily from~\eqref{mysharper}. 
}


\subsection{Proof of \thmref{trilateration}}
\label{sec:proof_trilateration}
Let $\bar a$ denote the average dissimilarity vector defined in \algref{trilateration} based on $Y$, and define $\bar b$ similarly based on $Z$.
Let $\Theta$ denote the matrix of dissimilarities between $\tY$ and $Z$, and let $\hat Y$ denote the result of \algref{trilateration} with inputs $Z$ and $\Theta$.  
From \algref{trilateration}, we have 
\beq
\tY^\top = \frac12 Y^\ddag (\bar a \one^\top - \tDelta^\top), \quad
\hat Y^\top = \frac12 Z^\ddag (\bar b \one^\top  -\Theta^\top), \quad
\tilde Z^\top = \frac12 Z^\ddag (\bar b \one^\top - \tilde\Lambda^\top),
\eeq 
due to the fact that the algorithm is exact.

We have
\beq
\|\tilde Z - \tY\|_2 \le \|\tilde Z - \hat Y\|_2 + \|\hat Y - \tY\|_2.
\eeq
On the one hand, 
\beq
2 \|\tilde Z - \hat Y\|_2 
\le \|Z^\ddag\| \|\tilde\Lambda - \Theta\|_2 
\le \|Z^\ddag\| (\|\tilde\Lambda - \tDelta\|_2 + \|\tDelta - \Theta\|_2). 
\eeq
On the other hand, starting with the triangle inequality,
\begin{align*}
2 \|\hat Y - \tY\|_2
&= \|Z^\ddag (\bar b \one^\top -\Theta^\top) - Y^\ddag (\bar a \one^\top -\tDelta^\top)\|_2 \\
&\le \|Z^\ddag (\bar b \one^\top -\Theta^\top) - Z^\ddag (\bar a \one^\top -\tDelta^\top)\|_2 + \|Z^\ddag (\bar a \one^\top -\tDelta^\top) - Y^\ddag (\bar a \one^\top -\tDelta^\top)\|_2 \\
&\le \|Z^\ddag\| (\|\bar b \one^\top- \bar a \one^\top\|_2 + \|\Theta -\tDelta\|_2) + \|\bar a \one^\top -\tDelta^\top\| \|Z^\ddag - Y^\ddag\|_2.
\end{align*}
Together, we find that 
\beq
2 \|\tilde Z - \tY\|_2 \le \|Z^\ddag\| (\|\tilde\Lambda - \tDelta\|_2 + 2 \|\Theta -\tDelta\|_2 + \sqrt{m} \|\bar b - \bar a\|) +  \|\bar a \one^\top -\tDelta^\top\| \|Z^\ddag - Y^\ddag\|_2.
\eeq

In the following, we bound the terms $\|\bar a \one^\top -\tDelta^\top\|$, $\|\Theta -\tDelta\|_2$ and $\|\bar b - \bar a\|$, separately. 

First, using \lemref{pseudo} and the fact that $(Y^\ddag)^\ddag = Y$ has full rank, 
\beq 
\|\tY\| = \frac12 \|Y^\ddag (\bar a \one^\top -\tDelta^\top)\| \ge \frac12 \|Y\|^{-1} \|\bar a \one^\top -\tDelta^\top\|\,.
\eeq
Therefore,
\beq
\|\bar a \one^\top -\tDelta^\top\| \le 2 \|Y\| \|\tY\|.
\eeq

Next, set $Y = [y_1, \cdots, y_m]^\top$ and $Z = [z_1, \cdots, z_m]^\top$, as well as $\tY = [\ty_1, \cdots, \ty_n]^\top$.
Since 
\beq
(\Theta - \tDelta)_{ij} = 2 \ty_i^\top (y_j - z_j) + \|z_j\|^2 - \|y_j\|^2,
\eeq
we have
\begin{align}
\|\Theta -\tDelta\|_2 
= \|2 \tY (Y^\top - Z^\top) + \one c^\top\|_2
\le 2 \|\tY\| \|Y - Z\|_2 + \sqrt{m} \|c\|,
\end{align}
with $c = (c_1, \dots, c_m)$ and $c_j = \|z_j\|^2 - \|y_j\|^2$.
Note that
\begin{align*}
\|c\|^2 
&= \sum_{j \in [m]} (\|z_j\|^2 - \|y_j\|^2)^2 \\
&\le \sum_{j \in [m]} \|z_j - y_j\|^2 (\|z_j\| + \|y_j\|)^2 \\
&\le (\rho_\infty(Y) + \rho_\infty(Z))^2 \|Z - Y\|_2^2,
\end{align*}
so that
\beq
\|\Theta -\tDelta\|_2 
\le 2 \|\tY\| \|Y - Z\|_2 + \sqrt{m} (\rho_\infty(Y) + \rho_\infty(Z)) \|Z - Y\|_2.
\eeq

Finally, recall that $\bar a$ and $\bar b$ are respectively the average of the columns of the dissimilarity matrix for the landmark $Y$ and the landmark $Z$.  Using the fact that the $y$'s are centered and that the $z$'s are also centered, we get 
\ajr{
\beq
\bar b - \bar a = c + \avgc \one,
\eeq
}
where \ajr{$\avgc = \frac1m \sum_{j \in [m]} c_j$, and therefore
\begin{align}
\|\bar b - \bar a\|^2 
\le \sum_{j \in [m]} (c_j + \avgc)^2
= \|c\|^2 + 3m \avgc^2
\le 4\|c\|^2\,,
\end{align}
using the Cauchy-Schawrz inequality at the last step.}

Combining all these bounds, we obtain the bound stated in \eqref{trilateration}.  The last part comes from the triangle inequality and an application of \lemref{wedin}.

\subsection{Proof of Corollary~\ref{cor:classical scaling}}\label{proof:cor:classical scaling}
\ajr{
If the half-width $\omega= 0$, the claim becomes trivial. Hence, we assume $\omega>0$, which implies that $Y = [y_1\dotsc y_m]^\top\in \bbR^{m\times d}$ is of rank $d$.
Recall that $\nu_d(Y)$ denotes the $d$-th largest singular value of $Y$. 
By characterization~\eqref{with} and since $Y$ has full column rank, we have $\nu_d(Y) = \sqrt{m} \omega$.

We denote by $\Lambda = (\lambda_{ij})\in \bbR^{m\times m}$ and $\Delta = (\delta_{ij})\in \bbR^{m\times m}$ and represent the centering matrix of size $m$, by $H$. Using \eqref{norm} and the fact that $\|H\|_\infty = 1$ (since $H$ is an orthogonal projection), we have
\beq\label{eq:eps-eta0}
\eps_0^2 \equiv \tfrac12 \|H(\Lambda - \Delta)H\| \le \|\Lambda - \Delta\| \le \|\Lambda - \Delta\|_2  = m\eta^2 \,.
\eeq

By our assumption \ajR{$\frac{\eta}{\omega}\le \frac{1}{\sqrt{2}} <1$}, which along with \eqref{eq:eps-eta0} yields 
\beq\label{eps-nu}
\eps_0^2<  m\omega^2 = \nu_{d}^2(Y)\,.
\eeq
In addition, by~\eqref{eq:MDSidentity} and since $Y\one = 0$ (data points are centered), we have $YY^\top = HYY^\top H = -\frac12 H\Delta H$, and as a result $\nu_d(-\frac12 H\Delta H) = \nu_d^2(Y)$.
By using the Weyl's inequality, we have
\[
\nu_d(-\dfrac12 H\Lambda H) \ge \nu_d(-\dfrac12 H\Delta H) - \eps_0^2 = \nu_d^2(Y) - \eps_0^2>0\,,
\] 
where the last step holds by~\eqref{eps-nu}. In words, the first top $d$ eigenvalues of $(-1/2) H\Lambda H$ are positive. Therefore, if $Z= [z_1,\dotsc, z_m]^\top \in \bbR^{m\times d}$ is the output of the classical scaling with input $\Lambda$, we have that $ZZ^\top$ is indeed the best rank $d$- approximation of $(-1/2) H\Lambda H$. Given that $(-1/2) H\Delta H$ is of rank $d$, this implies that
\beq
 \|ZZ^\top +\tfrac12 H\Lambda H\|_2 \le \|\tfrac 12 H(\Lambda-\Delta) H\|_2\,. 
\eeq
Thus, by triangle inequality
\begin{align}
\eps^2 \equiv\|ZZ^\top - YY^\top\| &\le \|ZZ^\top +\tfrac12 H\Lambda H\| + \|\tfrac12 H(\Lambda- \Delta)H\| \nonumber\\
&\le  \|ZZ^\top +\tfrac12 H\Lambda H\|_2 + \|\tfrac12 H(\Lambda- \Delta)H\|_2 \nonumber\\
&\le\|H(\Lambda - \Delta)H\|_2 \nonumber\\
&\le \|\Lambda - \Delta\|_2 \nonumber\\
&\le  m\eta^2 \le \ajR{m\omega^2/2}\,.\label{eq:eps-eta}
\end{align}
where in the penultimate line we used \eqref{norm} and the fact that $\|H\|_\infty = 1$. The last line follows from the definition of $\eta$ and our assumption on $\eta$, given in the theorem statement.

We next apply \thmref{procrustes} with $p=\infty$.
Note that by invoking Equations~\eqref{radius} and \eqref{with}, we get
\ajR{
\beq
\|Y^\ddag\| \eps  = \frac{\eps}{\sqrt{m}\omega} \le \frac{1}{\sqrt{2}}\,,
\eeq 
}
Hence, by using \thmref{procrustes} we have
\begin{align}
\min_{Q \in \cO} \bigg(\frac1m \sum_{i \in [m]} \|z_i - Q y_i\|^2\bigg)^{1/2} &\le \sqrt{\frac{d}{m}} \min_{Q\in \cO} \|Z-YQ\| \nonumber\\
& \le \sqrt{\frac{d}{m}} (\rho/\omega+2) \frac{\eps^2}{\sqrt{m}\omega}\nonumber\\
&\le \sqrt{d}(\rho/\omega+2) \frac{\eta^2}{\omega} \le \frac{3\sqrt{d}\rho \eta^2}{\omega^2}\,,
\end{align}
where the last line follows from~\eqref{eq:eps-eta} and the fact that $\omega\le \rho$.
}
\subsection{Proof of Corollary~\ref{cor:trilateration}}\label{proof:cor:trilateration}
\new{We apply \thmref{trilateration} to $\tY = [\ty_1, \cdots, \ty_n]^\top$, $Y = [y_1, \cdots, y_m]^\top$, and $Z = [z_1, \cdots, z_m]^\top$.
To be in the same setting, we need $Z$ to have full rank.  As we point out in \remref{trilateration}, this is the case as soon as $\|Y^\ddag\| \|Z-Y\| < 1$.  Since $\|Y^\ddag\| = (\sqrt{m} \omega)^{-1}$ and $\|Z-Y\| \le \sqrt{m} \max_{i\in [m]} \|z_i-y_i\| \le \sqrt{m} \eps$, the condition is equivalent to $\eps < \omega$, which is fulfilled by assumption.  
}
Continuing, we have
\beq
\|Z-Y\|_2\le \sqrt{m} \max_{i\in [m]} \|z_i-y_i\|\le \sqrt{m}\epsilon.
\eeq 
Hence, by~\eqref{trilateration-aux1}, 
\beq
\|Z^\ddag - Y^\ddag\|_2 \le \frac{\frac{2}{m\omega^2} \sqrt{m}\epsilon}{(1 -\frac{1}{\sqrt{m}\omega} \sqrt{m}\epsilon)_+^2} 
\le \frac{8 \eps}{\sqrt{m}\omega^2} 
\le \frac{4}{\sqrt{m}\omega}\,,
\eeq
using the fact that $\eps/\omega\le 1/2$.
Hence,
\beq
\|Z^\ddag\|  \le \|Y^\ddag\| + \|Z^\ddag - Y^\ddag\|
\le \frac{5}{\sqrt{m}\omega}\,, 
\eeq
 Further, $\|\tDelta-\tLambda\|_2 =  \sqrt{mn}\eta^2$. In addition, $\|\tY\| \le \sqrt{n} \zeta$. Likewise, $\|Y\| \le \sqrt{m}\rho$.
Therefore, by applying \thmref{trilateration}, we get
\begin{align}
\|\tZ - \tY\|_2 
&\le \frac{5}{\sqrt{m}\omega} \Big[\frac{1}{2}\sqrt{nm} \eta^2 + 2\sqrt{nm}\zeta \epsilon + \ajr{2} \sqrt{m} (2\rho_\infty + \epsilon) \sqrt{m} \epsilon \Big] 
+ (\sqrt{m} \rho)(\sqrt{n} \zeta) {\frac{8 \eps}{\sqrt{m}\omega^2}}\nonumber\\
&\le \ajr{20}\left( \frac{\sqrt{n}\eta^2}{\omega} + \frac{\sqrt{n} \zeta\epsilon}{\omega} + \frac{\rho_\infty+\epsilon}{\omega} \sqrt{m}\epsilon + \frac{\sqrt{n}\rho\zeta \eps}{\omega^2}\right)\,,
\end{align}
from which we get the stated bound, using the fact that $\eps \le \omega \le \rho \le \rho_\infty$.

\subsection{Proof of Corollary~\ref{cor:landmark isomap}}\label{proof:cor:landmark isomap}
Without loss of generality, suppose the chosen landmark points are $x_1, \dots, x_\ell$.  Using $\{\gamma_{ij}: i,j \in [\ell]\}$, we embed them using classical scaling, obtaining a centered point set $z_1, \dots, z_\ell \in \bbR^d$.  
\ajR{Note that by our assumption on the number of landmarks $\ell\ge 1$, we have
\[
\xi < (72\sqrt{d})^{-1} (\rho/\omega_*)^{-3} < \frac{1}{24} ({\rho}/{\omega_*})^{-2}\,,
\]
since $\omega_*\le \rho$. Hence the assumption on $\xi$ in \corref{isomap} holds and by applying this corollary,} we have
\beq
\min_{Q \in \cO} \bigg(\frac1\ell \sum_{i \in [\ell]} \|z_i - Q y_i\|^2\bigg)^{1/2} 
\le \frac{36 \sqrt{d} \rho_*^3}{\omega_*^2} \xi,
\eeq
where $\rho_*$ and $\omega_*$ are the max-radius and half-width of $\{y_1,\dotsc, y_\ell\}$.
We may assume that the minimum above is attained at $Q = I$ without loss of generality, in which case we have
\beq
\eps \equiv \max_{i \in [\ell]} \|z_i - y_i\| 
\le \frac{36 \sqrt{d \ell} \rho^3}{\omega_*^2} \xi,
\eeq
using the fact that $\rho_* \le \rho$.

The next step consists in trilaterizing the remaining points based on the embedded landmarks.  With $\eta$ as in \eqref{eta}, and noting that $\eps/\omega_* \le 1/{2}$ by our assumption on $\xi$, we may apply \corref{trilateration} (with the constant $C_0$ defined there) to obtain
\begin{align}
\frac1{C_0} \bigg(\frac1{n-\ell} \sum_{i = \ell+1}^n \|\tz_i - \ty_i\|^2\bigg)^{1/2} 
&\le \frac{\eta^2}{\omega_*} + \left[\frac{\rho_*\rho}{\omega_*^2} + \frac{\sqrt{\ell} \rho_*}{\sqrt{n-\ell}\,\omega_*}\right] \eps \\
&\le \frac{\eta^2}{\omega_*} + \frac{2 \rho^2 \eps}{\omega_*^2} \\
&\asymp \frac{\rho^2 \xi}{\omega_*} +  \frac{\rho^2}{\omega_*^2} \frac{\sqrt{d \ell} \rho^3 \xi}{\omega_*^2} \\
&\asymp \frac{\sqrt{d} \rho^5}{\omega_*^4} \sqrt{\ell}\, \xi\,,
\end{align}
using the fact that $\omega_* \le \rho_*$.

With this and the fact that 
\beq
\bigg(\frac1\ell \sum_{i =1}^\ell \|z_i - y_i\|^2\bigg)^{1/2} \le \eps,
\eeq
along with the bound on $\eps$, we have
\ajR{
\[
\min_{Q \in \cO} \bigg(\frac1n \sum_{i \in [\ell]} \|z_i - Q y_i\|^2 + \frac1n \sum_{i \in [n-\ell]} \|\tz_i - Q \ty_i\|^2\bigg)^{1/2} \lesssim \frac{\sqrt{d} \rho^5}{\omega_*^4} \sqrt{\ell}\, \xi
\asymp \frac{\rho^2}{\omega_*}\,,
\]
using our assumption on the number of landmarks.}



\appendix
\section*{Appendix}
\renewcommand{\thesubsection}{A.\arabic{subsection}}\subsection{A succinct proof that \algref{trilateration} is correct}
\label{sec:app_trilateration}
To prove that \algref{trilateration} is exact, it suffices to do so for the case where we want to position one point, i.e., when $n = 1$, and we denote that point by $\ty$.  In that case, $\tDelta$ is in fact a (row) vector, which we denote by $\tdelta^\top$.
We have $\|\ty - y_i\|^2 = \|\ty\|^2 + \|y_i\|^2 - 2 y_i^\top \ty$, so that $\delta = \|\ty\|^2 \one + \zeta - 2 Y \ty$, where $\zeta = (\|y_1\|^2, \dots, \|y_m\|^2)^\top$.  
We also have $\|y_j - y_i\|^2 = \|y_j\|^2 + \|y_i\|^2 - 2 y_j^\top y_i$, so that $\bar a = b \one + \zeta$, where $b = \frac1m (\|y_1\|^2 + \cdots + \|y_m\|^2)$, using the fact that $\frac1m \sum_{i=1}^m y_i = 0$.
Hence, $\bar a - \tdelta = (b - \|\ty\|^2) \one + 2 Y \ty$, and therefore, 
\beq
\frac12 Y^\ddag (\bar a - \tdelta) = \frac12 (b - \|\ty\|^2) Y^\ddag \one + Y^\ddag Y \ty.
\eeq
We now use the fact that $Y^\ddag = (Y^\top Y)^{-1} Y^\top$.
On the one hand, $Y^\ddag \one = (Y^\top Y)^{-1} Y^\top \one = 0$ since $Y^\top \one = 0$ (because the point set is centered).
On the other hand, $Y^\ddag Y = (Y^\top Y)^{-1} Y^\top Y = I$.
We conclude that $\frac12 Y^\ddag (\bar a - \tdelta) = \ty$, which is what we needed to prove.

\subsection{Proof of \prpref{better BSLT}}\label{sec:proof better BSLT}

The data points are denoted $x_1, \dots, x_n \in \M$, and by assumption we assume that $x_i = \varphi(y_i)$, where $\varphi : \cD \to \M$ is a one-to-one isometry, with $\cD$ being a convex subset of $\bbR^d$.  Fix $i, j \in [n]$, and note that $g_{ij} = g_\cM(x_i, x_j) = \|y_i - y_j\|$.  

If $g_{ij} \le r$, then $\|x_i - x_j\| \le g_{ij} \le r$, so that $i$ and $j$ are neighbors in the graph, and in particular $\gamma_{ij} = \|x_i - x_j\|$.  We may thus conclude that, in this situation, $\gamma_{ij} \le g_{ij}$, which implies the stated bound.

Henceforth, we assume that $g_{ij} > r$.
Consider $z_k = y_i + (k/m) (y_j - y_i)$, where $m = \lceil 2 g_{ij}/r \rceil \ge 2$.  Note that $z_0 = y_i$ and $z_m = y_j$.
Let $y_{i_k}$ be the closest point to $z_k$ among $\{y_1, \dots, y_n\}$, with $i_0 = i$ and $i_m = j$.  
By the triangle inequality, we have
\begin{align}
\|y_{i_{k+1}} - y_{i_k}\| 
&\le \|z_{k+1} - z_k\| + \|y_{i_{k+1}} - z_{k+1}\| + \|y_{i_k} - z_k\| \\
&\le \frac1m g_{ij} + 2a
\le r/2 + 2a
\le r,
\end{align}
if $a/r \le 1/4$.
Therefore, 
\beq
\|x_{i_{k+1}} - x_{i_k}\| 
\le g_\M(x_{i_{k+1}}, x_{i_k})
= \|y_{i_{k+1}} - y_{i_k}\|
\le r,
\eeq
implying that $(i_k : k = 0, \dots, m)$ forms a path in the graph.  

So far, the arguments are the same as in the proof of \cite[Thm 2]{bernstein2000graph}.  What makes our arguments sharper is the use of the Pythagoras theorem below.  To make use of that theorem, we need to construct a different sequence of points on the line segment. 
Let $\tilde z_k$ denote the orthogonal projection of $y_{i_k}$ onto the line (denoted $\cL$) defined by $y_i$ and $y_j$.  
See \figref{line} for an illustration.

\begin{figure}[t!]
	\centering
    \includegraphics[scale = 0.5]{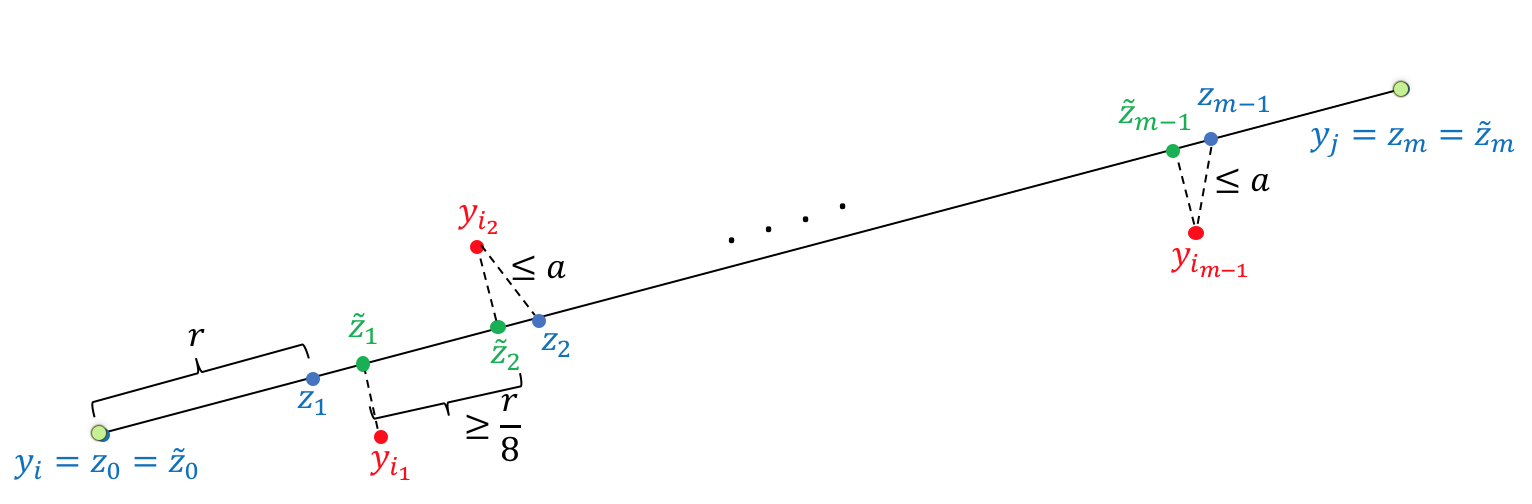}
	\caption{illustration for the proof of \prpref{better BSLT}}
	\label{fig:line}
\end{figure}

In particular the vector $\tilde z_k - y_{i_k}$ is orthogonal to $\cL$, and 
\beq
\|\tilde z_k - y_{i_k}\| = \min_{z \in \cL} \|z - y_{i_k}\| \le \|z_k - y_{i_k}\| \le a.
\eeq
It is not hard to see that $\tilde z_k$ is in fact on the line segment defined by $y_i$ and $y_j$.  Moreover, they are located sequentially on that segment.  Indeed, using the triangle inequality,
\begin{align}
\|\tilde z_k - y_i\| 
&\le \|z_k - y_i\| + \|z_k - \tilde z_k\| \\
&\le \|z_k - y_i\| + \|z_k - y_{i_k}\| + \|y_{i_k} - \tilde z_k\| \\
&\le \|z_k - y_i\| + 2 a \\
&= \frac{k}m g_{ij} + 2a,
\end{align}
while, similarly,
\begin{align}
\|\tilde z_{k+1} - y_i\| 
\ge \|z_{k+1} - y_i\| - 2 a
= \frac{k+1}m g_{ij} - 2a,
\end{align}
so that $\|\tilde z_k - y_i\| < \|\tilde z_{k+1} - y_i\|$ as soon as $g_{ij}/m > 4 a$.  Noting that $g_{ij} > (m-1)r/2$, this condition is met when $a/r \le (m-1)/8m$.  Recalling that $m \ge 2$, it is enough that $a/r \le 1/16$.
From the same derivations, we also get
\beq\label{z diff lb}
\|\tilde z_{k+1} - \tilde z_k\| 
\ge \frac1m g_{ij} - 4a
\ge \frac{(m-1)r}{2m} - 4a
\ge r/8,
\eeq
if $a/r \le 1/32$.

Since $(i_k : k = 0, \dots, m)$ forms a path in the graph, we have
\beq
\gamma_{ij} 
\le \sum_{k=0}^{m-1} \|x_{i_{k+1}} - x_{i_k}\|
\le \sum_{k=0}^{m-1} \|y_{i_{k+1}} - y_{i_k}\|.
\eeq
By the Pythagoras theorem, we then have 
\begin{align}
\|y_{i_{k+1}} - y_{i_k}\|^2 
&= \|\tilde z_{k+1} - \tilde z_k\|^2 + \|y_{i_{k+1}} - \tilde z_{k+1} + \tilde z_k - y_{i_k}\|^2 \\
&\le \|\tilde z_{k+1} - \tilde z_k\|^2 + (2 a)^2,
\end{align}
so that, using \eqref{z diff lb},
\beq
\|y_{i_{k+1}} - y_{i_k}\|
\le (1 + (2a)^2/(r/8)^2)^{1/2} \|\tilde z_{k+1} - \tilde z_k\|
= (1 + C (a/r)^2) \|\tilde z_{k+1} - \tilde z_k\|,
\eeq
where $C \le 128$, yielding
\beq
\gamma_{ij} 
\le (1 + C (a/r)^2) \sum_{k=0}^{m-1} \|\tilde z_{k+1} - \tilde z_k\|
= (1 + C (a/r)^2) g_{ij}. 
\eeq

\subsection{Proof of \prpref{random landmark}}\label{sec:proof random landmark}
We use concentration bounds for random matrices developed by \citet{tropp2012user}.
Consider a point set $\cY = \{y_1, \dots, y_n\}$, assumed centered without loss of generality.  We apply {\sf Random} to select a subset of $\ell$ points chosen uniformly at random with replacement from $\cY$.  We denote the resulting (random) point set by $\cZ = \{z_1, \dots, z_\ell\}$.
Let $Y = [y_1 \cdots y_n]$ and $Z = [z_1 \cdots z_\ell]$.
We have that $\cY$ has squared half-width equal to $\omega^2 \equiv \nu_d(Y^\top Y)/n$, and similarly, $\cZ$ has squared half-width equal to $\omega_Z^2 = \nu_d(Z^\top Z - \ell\, \bar z \bar z^\top)/\ell$, where $\bar z = (z_1 + \cdots + z_\ell)/\ell$. 
Note that, by \eqref{mirsky},
\beq
\omega_Z^2 
\ge \nu_d(Z^\top Z)/\ell - \nu_1(\bar z \bar z^\top)
= \nu_d(Z^\top Z)/\ell - \nu_1(\bar z)^2
= \nu_d(Z^\top Z)/\ell - \|\bar z\|^2.
\eeq
We bound the two terms on the right-hand side separately.

First, we note that $Z^\top Z = \sum_j z_j z_j^\top$, with $z_1 z_1^\top, \dots, z_\ell z_\ell^\top$ sampled independently and uniformly from $\{y_1 y_1^\top, \dots, y_n y_n^\top\}$.  These matrices are positive semidefinite, with expectation $Y^\top Y/n$, and have operator norm bounded by $\max_i \|y_i y_i^\top\| = \max_i \|y_i\|^2 = \rho_\infty^2$.  
We are thus in a position to apply \cite[Thm 1.1, Rem 5.3]{tropp2012user}, which gives that
\beq
\P\left(\nu_d(Z^\top Z)/\ell \le \tfrac12 \omega^2\right)
\le d \exp\big[-\tfrac18 \ell \omega^2/\rho_\infty^2\big].
\eeq

Next, we note that $\ell \bar z = \sum_j z_j$, with $z_1, \dots, z_n$ being iid uniform in $\{y_1, \dots, y_n\}$.  These are here seen as rectangular $d \times 1$ matrices, with expectation 0 (since the $y$'s are centered), and operator norm bounded by $\max_i \|y_i\| = \rho_\infty$.
We are thus in a position to apply \cite[Thm 1.6]{tropp2012user}, which gives that, for all $t \ge 0$,
\beq
\P\left(\|\bar z\| \ge t/\ell \right)
\le (d+1) \exp\big[-t^2/(2 \sigma^2 + \tfrac13 \rho_\infty t)\big],
\eeq
where 
\beq
\sigma^2 
= (\ell/n) \big(\|Y^\top Y\| \vee {\textstyle\sum}_i \|y_i\|^2\big)
= (\ell/n) {\textstyle\sum}_i \|y_i\|^2
\le \ell \rho_\infty^2.
\eeq
In particular,
\begin{align}
\P\left(\|\bar z\| \ge \tfrac14 \omega^2\right)
&\le (d+1) \exp\big[-\tfrac14\omega^2 \ell^2/(2 \rho_\infty \ell + \tfrac13 \rho_\infty \tfrac12 \omega \ell)\big] \\
&\le (d+1) \exp\big[-\tfrac19 \ell \omega^2/\rho_\infty^2 \big]\,,
\end{align}
using in the last line the fact that $\omega \le \rho_\infty$.

Combining these inequalities using the union bound, we conclude that
\beq
\P\left(\omega_Z \le \tfrac12 \omega\right)
\le d \exp\big[-\tfrac18 \ell \omega^2/\rho_\infty^2\big] 
	+ (d+1) \exp\big[-\tfrac19 \ell \omega^2/\rho_\infty^2 \big],
\eeq
from which the stated result follows.



\subsection*{Acknowledgements}
We are grateful to Vin de Silva, Luis Rademacher, and Ilse Ipsen for helpful discussions and pointers to the literature.
Part of this work was performed while the first and second authors were visiting the Simons Institute\footnote{ The Simons Institute for the Theory of Computing (\url{https://simons.berkeley.edu})} on the campus of the University of California, Berkeley.
The first author was partially supported by the National Science Foundation (DMS 0915160, 1513465, 1916071) and the French National Research Agency (ANR 09-BLAN-0051-01).  \ajr{The second author was partially supported by an Outlier Research in Business (iORB) grant from the USC Marshall School of Business, a Google Faculty Research Award and the NSF CAREER Award DMS-1844481.}

\small
\bibliographystyle{abbrvnat}
\bibliography{ref}

\end{document}